\documentclass[lettersize,journal]{IEEEtran}
\usepackage{amsmath,amsfonts}
\usepackage{algorithmic}
\usepackage{algorithm}
\usepackage{array}
\usepackage[caption=false,font=normalsize,labelfont=sf,textfont=sf]{subfig}
\usepackage{textcomp}
\usepackage{stfloats}
\usepackage{url}
\usepackage{verbatim}
\usepackage{graphicx}
\usepackage{cite}
\usepackage{adjustbox}
\usepackage{amssymb,subfloat,booktabs,makecell,bbding,multirow,tabularx,ragged2e}
\usepackage[dvipsnames]{xcolor}
\RequirePackage[colorlinks,citecolor=RoyalBlue]{hyperref}

\usepackage{tikz, pifont}

\newcommand{\onedot}{.}
\def\eg{\emph{e.g}\onedot} 
\def\ie{\emph{i.e}\onedot} 
\def\cf{\emph{cf}\onedot} 
\def\etc{\emph{etc}\onedot} 
 
\def\iid{i.i.d\onedot}

\def\id{\mathcal{D}_{\text{ID}}}
\def\idimage{\mathcal{X}_{\text{ID}}}
\def\idlabel{\mathcal{Y}_{\text{ID}}}

\def\ood{\mathcal{D}_{\text{OOD}}}
\def\oodimage{\mathcal{X}_{\text{OOD}}}

\usepackage{bm}
\usepackage{mathtools, amsthm}
\usepackage{cleveref}
\newcommand{\mtx}[1]{\bm{#1}}
\def\mb{\mathbb}
\newcommand{\defeq}{\vcentcolon=}
\newcommand{\dotp}[2]{\left\langle #1, #2\right\rangle}
\newcommand{\paren}[1]{\left(#1\right)}
\newcommand{\brkt}[1]{\left[#1\right]}

\newtheorem{assumption}{Assumption}
\newtheorem{theorem}{Theorem}
\newtheorem{theoremappendix}{Theorem}
\newtheorem*{remark}{Remark}

\begin{document}

\newcommand{\mytitle}{Revisiting Energy-Based Model for Out-of-Distribution Detection}

\title{\mytitle}

\author{Yifan Wu,~\IEEEmembership{Student Member,~IEEE},
        Xichen Ye,
        Songmin Dai,
        Dengye Pan, 
        Xiaoqiang Li,~\IEEEmembership{Member,~IEEE},\\
        Weizhong Zhang,
        and~Yifan Chen,~\IEEEmembership{Member,~IEEE}
\thanks{Yifan Wu and Xichen Ye are with the School of Computer Engineering and Science, Shanghai University, Shanghai, China,
and also with the School of Computer Science, Fudan University, Shanghai, China (e-mail: victorwu001219@gmail.com; yexichen0930@outlook.com).}
\thanks{Songmin Dai, Dengye Pan, and Xiaoqiang Li are with the School of Computer Engineering and Science, Shanghai University, Shanghai, China (e-mail: laodar@shu.edu.cn; pandy@shu.edu.cn; xqli@shu.edu.cn).}
\thanks{Weizhong Zhang is with the School of Data Science, Fudan University, Shanghai, China (e-mail: weizhongzhang@fudan.edu.cn).}
\thanks{Yifan Chen is with the Department of Computer Science, Hong Kong Baptist University, Hong Kong, China (e-mail: yifanc@hkbu.edu.hk).}
\thanks{Yifan Wu and Xichen Ye contributed equally to this work.}
\thanks{Corresponding authors: Xiaoqiang Li and Yifan Chen.}
}

\markboth{Journal of \LaTeX\ Class Files,~Vol.~14, No.~8, August~2021}%
{Shell \MakeLowercase{\textit{et al.}}: A Sample Article Using IEEEtran.cls for IEEE Journals}

\IEEEpubid{0000--0000/00\$00.00~\copyright~2021 IEEE}

\maketitle

\begin{abstract}
Out-of-distribution (OOD) detection is an essential approach to robustifying deep learning models, enabling them to identify inputs that fall outside of their trained distribution. 
Existing OOD detection methods usually depend on crafted data, such as specific outlier datasets or elaborate data augmentations; this characteristic is reasonable while the frequent mismatch between crafted data and OOD data limits model robustness and generalizability. 
In response to this issue, we introduce Outlier Exposure by Simple Transformations (OEST), a framework that enhances OOD detection by leveraging ``peripheral-distribution'' (PD) data;
specifically, PD data are samples generated through \emph{simple data transformations}, thus an efficient alternative to manually curated outliers.

We further adopt the energy-based models (EBMs) to study PD data.
We first recognize the ``energy barrier'' in OOD detection which characterizes the energy difference between in-distribution (ID) / OOD samples and eases the detection;
the in-between PD data are introduced to establish the energy barrier in training.
Furthermore, this energy barrier concept motivates a theoretically grounded energy-barrier loss to replace the classical energy-bounded loss, 
which leads to an improved paradigm, OEST*, and brings a more effective and theoretically sound separation between ID and OOD samples. 
We perform empirical validation to provide sanity checks of our proposal, and extensive experiments across various benchmarks demonstrate that OEST* achieves better or similar accuracy compared with state-of-the-art methods.
The source code of our method is available at: \href{https://github.com/victor-yifanwu/Outlier-Exposure-by-Simple-Transformations}{https://github.com/victor-yifanwu/Outlier-Exposure-by-Simple-Transformations}.

\end{abstract}


\section{Introduction}
\label{sec: intro}

\IEEEPARstart{T}{he} predominant assumption in model training is that test data are drawn independently and identically distributed ($\iid$) from the same distribution as the training data. 
Such distribution alignment is generally known as in-distribution (ID). 
Although the ID assumption leads to simple formulation, 
it rarely holds in \emph{open-world} scenarios as distribution shifts inevitably exist between training and testing data. 
This discrepancy poses significant challenges to a few existing models\cite{outlierhighd01sigmod,outliersurvey04aireview,outlier05handbook,outlierprogress19ieee}. It is essential to recognize these deviations as outliers, namely out-of-distribution (OOD)  samples\cite{amodei2016concrete,hendrycks2017abaseline,dietterich2017steps,leike2017ai,smuha2019eu,shneiderman2020ethicsai,mohseni2021mlsafety,hendrycks2021unsolved,hendrycks2022xrisk}, instead of blindly categorizing unseen samples into known classes with high confidence\cite{nguyen2015deep,hein2019relu}. Due to its broad application scenarios (\eg , autonomous vehicles\cite{jung2021standardized} and medical tasks\cite{gaspar2011systematic,hauskrecht2013outlier}), a number of methodologies for out-of-distribution detection have been developed\cite{conf/nips/YangWZZDPWCLSDZ22:openood,zhang2023openood}.

An intuitive approach to alleviate the misclassification of unknown samples as known categories, is to incorporate a significant amount of auxiliary external data in training~\cite{hendrycks2019oe, yu2019mcd, liu2020energy, mohseni2020pseudolabel, outliermining21ecml, abstention21arxiv, backgroundsampling20cvpr, papadopoulos2021oecc, ming2022posterior, Zhang_2023_WACV};
the utilization of real outliers generally brings about better performance. 
However, naturally the similarity between the crafted and the real OOD samples makes a significant difference, as revealed in recent studies~\cite{liznerski2022exposing, NEURIPS2023_LearningOE}.
This turns into an issue particularly in specialized fields such as medical or industrial imaging, where data characteristics greatly differ from those found in public vision datasets and high-quality external datasets are scarce \cite{Ruff2020AUR}. 
Even worse, another challenge arises in practice that the external auxiliary OOD dataset tends to contain quite a few ID samples, necessitating either elaborate algorithms or tremendous human labor to filter these samples. 
These complications deteriorate the effectiveness of OOD detection.

\IEEEpubidadjcol

In response to these issues, we manage to eliminate the reliance on real outliers in OOD detection,
through exploiting \emph{simple transformations} over ID samples; 
we further recognize and term the transformed samples as \emph{peripheral-distribution} (PD) data.
We consider this PD samples as neither ID nor regularly OOD. 
This ambiguity arises from the versatility of the collection of simple transformations: 
samples augmented with certain transformations, \eg, rotation \cite{gidaris2018unsupervised}, remain semantically ID, 
whereas those augmented with sobel filtering \cite{kanopoulos1988design} are distinct from ID samples. 
Experimental observation in Figure~\ref{fig:clusted-features} further confirms PD data is an interpolation between ID and OOD samples.

In addition, we revisit energy-based OOD detection~\cite[EBO]{liu2020energy} and connect the aforementioned concept of PD data to the energy-based model~\cite[EBM]{lecun2006tutorial}. 
We aim to encourage higher energy for samples from peripheral-distribution while lowering that of ID samples;
we therefore create an ``energy barrier'' between the ID and the PD data (and thus the OOD data).
The energy barrier we propose is formulated in \Cref{ass:periphery}, and we verify its statistical benefits in \Cref{thm:periphery}. 

Combining all the pieces above, we propose a new training paradigm, Outlier Exposure by Simple Transformations (\textbf{OEST}), which is illustrated in Figure~\ref{fig:energy barrier}. 
This approach features a comprehensive use of numerous data augmentation techniques, including those previously deemed non-contributory in~\cite{tack2020csi}. 
Remarkably, OEST achieves outstanding performance in both near-OOD and far-OOD detection tasks, with solely an extra $10$-epoch tuning.
The main contributions of this paper are summarized as follows: 
\begin{itemize}[leftmargin=*]
\item We introduce peripheral-distribution (PD) data for OOD detection, which consists of samples augmented through various simple transformations. 
\item We revisit the energy-based model, and suggest the energy polarization of ID and OOD samples can benefit OOD detection.
\item We propose to establish an energy barrier between ID and PD data, which consequently and provably enhance the distinction between ID and OOD samples.
\item We devise a targeted tuning strategy for existing classifiers built upon the establishment of energy barrier, which achieves state-of-the-art results with a large margin under both near- and far-OOD scenarios.
\end{itemize}

\begin{figure*}
    \centering
    \subfloat[Model learns better clustered features with OEST.]{
        \label{fig:clusted-features}
        \includegraphics[width=0.55\textwidth]{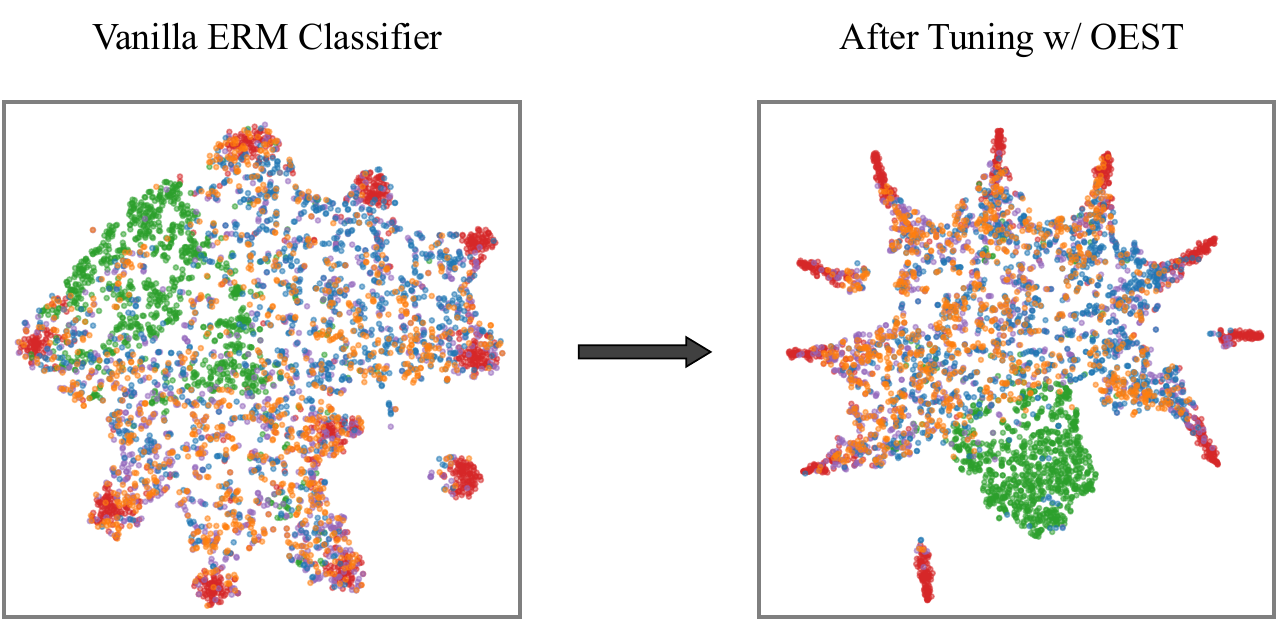}
        }
    \hspace{0.05\textwidth}
    \subfloat[Schematic of sample distribution and energy barriers.]{
        \label{fig:energy barrier}
        \includegraphics[width=0.35\textwidth]{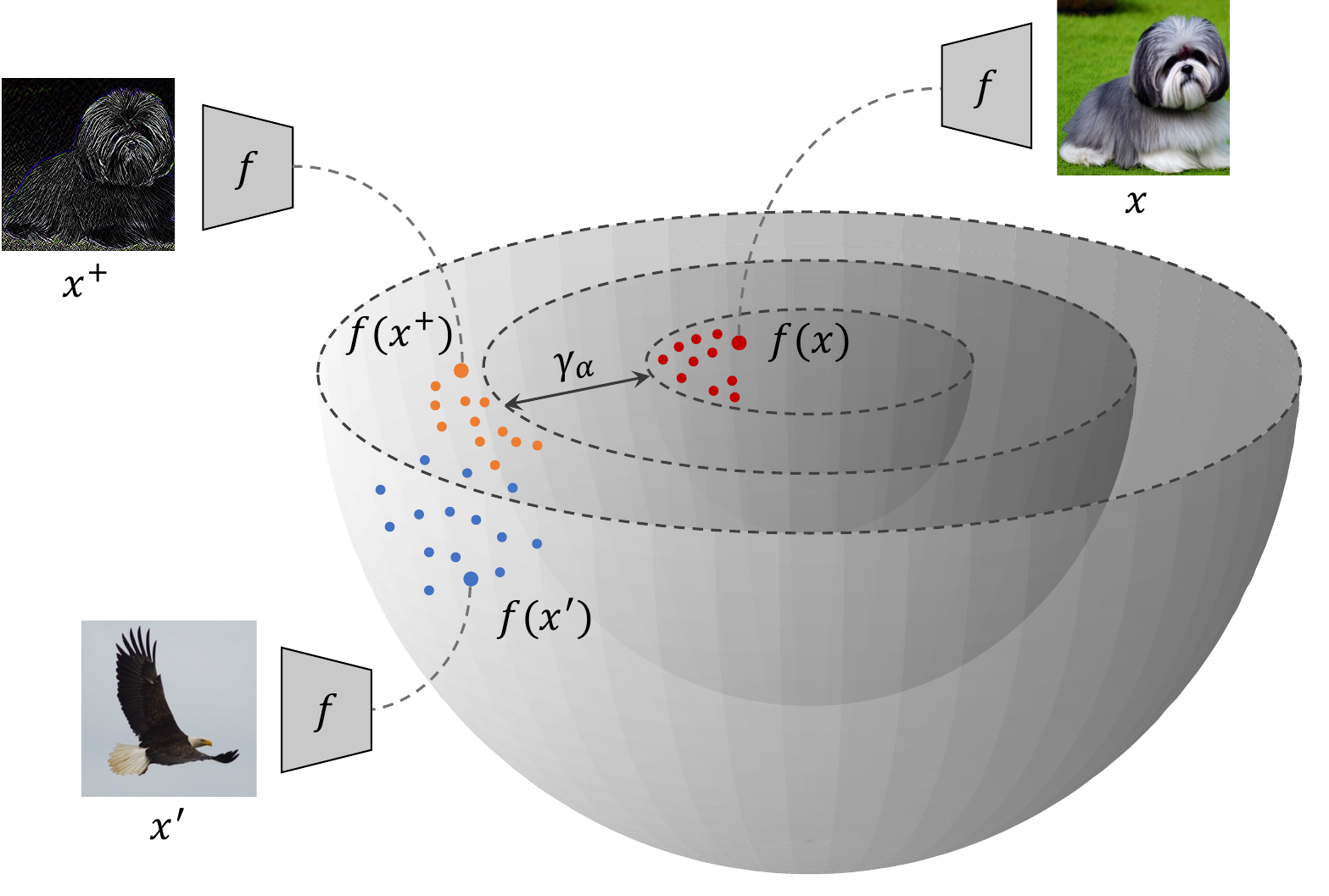}
        }
    \caption{
    (a) The t-SNE visualization of representations from CIFAR-10 (Red) test samples, rotated CIFAR-10 (Orange) test samples, CIFAR-100 (Blue), SVHN (Green) and ImageNet (Purple) before and after applying our training strategy OEST. 
    Specifically, the embedding features are extracted from the penultimate layer of ResNet-18 classifier (trained on CIFAR-10). \\
    (b) 
    We illustrate the feature space as a series of concentric spherical shells, where each shell corresponds to a certain energy level. 
    The innermost shell contains in-distribution samples with the lowest energy, represented by $\bm{x}$. 
    Moving outward, the orange points indicate augmented peripheral-distribution samples, denoted by $\bm{x}^{+}$. 
    OEST establishes an energy barrier (reading $\gamma_\alpha$) between ID ($\bm{x}$) and PD ($\bm x^+$) data, thus separating ID and out-of-distribution samples ($\bm{x}'$).
    }
\end{figure*}

The preliminary $4$-page version of this manuscript was presented in ICIP 2023 \cite{wu2023oest}, where we solely suggested applying simple transformations for OOD detection. 
In this extended paper, we provide a more comprehensive investigation into the proposed methodology;
we newly recognize the energy barrier between ID and PD data (in \Cref{sec: barriers}) to provably explain the empirical success of OEST, develop a theoretically rigorous energy loss function (in \Cref{sec: fine-tune}), and broaden the experimental evaluation across public benchmarks (in \Cref{sec:experiments}).

The rest of the paper is organized as follows. In Section~\ref{sec:related}, we review related works on OOD detection. 
In Section~\ref{sec: preliminaries}, we detail the preliminaries of this work, including its theoretical basis and a short introduction to the energy-based model.
In Section~\ref{sec:meth}, we illustrate the proposed method in detail, along with its theoretical analysis. 
Experimental results and analyses are provided in Section~\ref{sec:experiments}. Finally, we conclude the paper in Section~\ref{sec:conclu}. 

\section{Related Works}
\label{sec:related}

In this section, we review prior works on out-of-distribution (OOD) detection, involving three primary approaches: \ding{182}~OOD scoring methods (see Section~\ref{sec:2.1}), \ding{183}~training-based methods (see Section~\ref{sec:2.2}), and \ding{184}~methods with outlier exposure (see Section~\ref{sec:2.3}).

\subsection{OOD Scoring Methods} 
\label{sec:2.1}

In general, OOD scoring methods assess the likelihood that a sample originates from the training distribution, $\ie$, is in-distribution, based on sample features or model outputs. 

From a feature perspective, early studies employed parametric density estimation, assuming the feature embedding space consists of a mixture of multivariate Gaussian distributions, to score samples based on the Mahalanobis distance\cite{lee2018simple} or the gram matrix\cite{sastry2020gram}. 
A more recent approach~\cite{sun2022out} utilized the distance between the 
sample feature and its $k$-th nearest neighbor (KNN) as the score; 
likewise, SHE\cite{zhang2023she} leveraged the similarity between the 
sample feature and class centers 
for OOD detection.

For model outputs, one common OOD score is the maximum softmax prediction (MSP)\cite{hendrycks2017abaseline}. 
Subsequently, ODIN~\cite{liang2018odin} was proposed to utilize temperature scaling and input perturbation to maximize the MSP gap between ID and OOD data. 
Follow-up studies revealed that the key to ODIN's effectiveness is transforming the softmax score back to the logit space through temperature scaling; 
therefore, methods like the maximum logit scores\cite{hendrycks2022scaling} and the standardized max logits~\cite{jung2021standardized} were developed.

However, raw softmax or logit scores are found prone to overconfidence issues, which prompts the development of the energy-based models (EBM) \cite{liu2020energy, lin2021mood}. 
EBM employs an energy-based function to transform logits into a more reliable scoring metric and is theoretically underpinned via a likelihood perspective\cite{grathwohl2020your, morteza2022provable}.
Recent studies\cite{sun2021react, dong2022neural, djurisic2023ash, xu2024ish, wan2024nap} began to focus not only on the last or penultimate layers of the model but also on the hidden layers, as well as the activations among them. 

All the aforementioned OOD scoring methods (including ours) can be classified as \emph{post-hoc methods} because the OOD scores within, 
derived from feature or model outputs, can be implemented without modifying the training procedure or objective. 
These approaches avoid both the overhead cost of retraining and any detrimental impact on the ID accuracy of the original classifier.

\subsection{Training-based Methods} 
\label{sec:2.2}

Another genre of OOD detection intervenes in the model training.
We note, although methods with outlier exposure also fall under training-based approaches, we defer the related discussion to the next subsection and focus exclusively on the methods that do not utilize outlier samples in this subsection. 

Training-based approaches involve the following methodologies.
\ding{172}~From the perspective of confidence estimation, \cite{confbranch2018arxiv} proposed to modify the model structure, while \cite{wang2021energy} designed a new Softmax layer.
\ding{173}~Regarding modifying training objectives, G-ODIN \cite{hsu2020generalized} introduced a specialized objective called DeConf-C based on ODIN \cite{liang2018odin}, and \cite{wei2022mitigating} advocated for training with logit normalization (LogitNorm), a straightforward modification to the standard cross-entropy loss aimed at mitigating overconfidence.
\ding{174}~From the perspective of enhancing model representation, some studies have employed adversarial training\cite{good20nips, aloe20arxiv, hein2019relu, blur20iclr, outliermining21ecml} or stronger data augmentation\cite{mixup19nips, cutmix19cvpr, cutout17arxiv, augmix19arxiv, hendrycks2022pixmix} to enrich ID samples. 
\ding{175}~Additionally, self-supervised methods have been utilized to improve classifier robustness in OOD detection. 
\cite{transform18nips, hendrycks2019using} introduced an additional training objective, image transformation prediction, during model training. \ding{176}~Moreover, to enhance sensitivity to covariate shifts, \cite{tack2020csi} treats original and augmented samples as positive and negative examples, respectively. 
We remark this approach diverges from the traditional contrastive learning framework \cite{chen2020simple} by actively separating positive and negative samples in the feature space.

\subsection{Methods with Outlier Exposure} 
\label{sec:2.3}

As a broadly studied technique, outlier exposure in OOD detection can be divided into two main categories based on the source of outliers. 
\ding{172}~The first category utilizes a collected set of real-world OOD samples (\emph{real outliers}) to aid models in learning the discrepancy between ID and OOD, while \ding{173}~the second category focuses on \emph{generating outlier data} to enhance model robustness against various unforeseen OOD samples. 

\ding{172}~For the methodology based on real outliers, the initial approach was proposed by \cite{hendrycks2019oe}, which encourages high-entropic predictions on given outlier samples. 
Subsequently, MCD~\cite{yu2019mcd} employed a dual-branch network to distinguish between ID and OOD data, and a follow-up work\cite{liu2020energy} further tuned the classifier on both ID and OOD samples 
with energy-based loss (see more discussion on \Cref{sec: EBM}).
Other straightforward methods\cite{mohseni2020pseudolabel, outliermining21ecml, abstention21arxiv} treat the given OOD samples as the $(k+1)^\text{th}$-class. 
Recent studies\cite{backgroundsampling20cvpr, papadopoulos2021oecc, outliermining21ecml, ming2022posterior, Zhang_2023_WACV} started to focus on a selected set of meaningful outliers among numerous OOD samples. 

\ding{173}~For outlier generation, earlier studies tended to generate data based on low-dimensional feature spaces; 
specifically, they utilized KL divergence \cite{confcal18iclr}, low-density regions\cite{oodsg19nipsw}, high-confidence regions\cite{confgan18nipsw}, or meta-learning\cite{maml20nips}. 
Subsequent work instead proposes a new paradigm for generating outliers that can be implemented using both GANs and diffusion models\cite{dai2021mlad, dai2024generating}. 
Additionally, considering the challenges of image generation in the high-dimensional pixel space, 
recent approaches, such as VOS\cite{du2022vos} and NPOS\cite{tao2022npos}, 
have proposed to
generate outlier data by injecting perturbations into ID sample features.

In summary, OOD detection leveraging real outliers can achieve superior performance. However, the effectiveness of these methods can be significantly influenced by the correlations between the provided and the actual OOD samples \cite{lessbias19bmvc}.

\section{Preliminaries}
\label{sec: preliminaries}

In this section, we provide a formulation of out-of-distribution (OOD) detection in Section~\ref{sec: OOD}, followed by a revisit of the energy-based model (EBM) for OOD detection in Section~\ref{sec: EBM}.
For the reader's convenience, we list a collection of defined notations in Table~\ref{tab: notations}.

\begin{table}[!t]
\centering
\caption{Detailed Description of the Main Notations}
\label{tab: notations}
\begin{tabular}{ll}
\toprule
Notation & Definition \\ \midrule
$\id$ & joint distribution of ID data $(\mtx x, y)$ \\
$\ood$ & joint distribution of OOD data $(\bm{x}', y')$ \\
$\idimage$ & support space of ID inputs \\
$\oodimage$ & support space of  OOD inputs \\
$\mathcal{X}_\text{PD}$ & support space of PD inputs \\
$\idlabel$ & support space of ID labels \\
$f_\theta$ & the neural classifier with parameters $\theta$ \\ 
$\mathcal{L}_\text{CE}$ & cross-entropy loss function \\ 
$\mathcal{L}_\text{energy}$ & energy-based loss function \\ 
\bottomrule
\end{tabular}
\end{table}

\subsection{Out-of-Distribuion Detection}
\label{sec: OOD}

The emergence of out-of-distribution (OOD) detection was driven by the practical need for models to discern and reject inputs that are semantically different from the training distribution. 
To discuss this concept within a rigorous framework, we embrace the widely acknowledged definition of in-distribution data and out-of-distribution data, as outlined in the previous literature\cite{hendrycks2017abaseline, conf/nips/YangWZZDPWCLSDZ22:openood,zhang2023openood}.

In this paper, we consider a typical $C$-class classification problem, where we have access to independently and identically distributed (\textit{i.i.d.}) samples $(\bm{x}, y)$ drawn from the ID distribution, $\mathcal{D}_\text{ID}$.
Specifically, we denote the input space as $\idimage$, and the label space as $\idlabel = [C]$. 
In contrast, out-of-distribution (OOD) samples are drawn from a different distribution, $\mathcal{D}_\text{OOD}$.
For each OOD sample $(\bm{x}', y')$, the input $\bm{x}'$ has semantics that differ from those of any ID samples, and notably the label $y'$ does not belong to any of the $C$ classes present in the training dataset, \ie, $y'\notin\idlabel$. 
Typically, we denote the set of OOD inputs as $\oodimage$.

The goal of OOD detection is to design a score function:
\begin{equation}
    s_\theta(\bm x) \in \mathcal{R},
\end{equation}
where $\theta$ is the learnable parameters.
The desired outcome is that in-distribution samples receive higher scores than out-of-distribution samples, and consequently an OOD discriminator can be straightforwardly defined using this score function.

As a side note, one might consider modeling \( p(\bm{x}) \) using a generative model, and intuitively take \( p(\cdot) \) as a score function, given the strong modeling capabilities in modern generative modeling.
However, previous research has shown that the density functions estimated by deep generative models cannot be reliably used for OOD detection \cite{DBLP:conf/iclr/NalisnickMTGL19:DoGenKnow}.

Moreover, in this work we consider the scenario in which we only have access to a trained classifier, and our objective is to repurpose it as an OOD discriminator.
To address this practical constraint, energy-based models (EBMs) \cite{lecun2006tutorial} offer an alternative approach to constructing a score function to distinguish OOD samples with a classifier.
We will shortly detail EBM in the next subsection.

\subsection{Energy-Based Model}
\label{sec: EBM}

As an OOD scoring method (see \Cref{sec:2.1}), the essence of the energy-based model is to construct an energy function $E(\cdot, \cdot)$
that maps each sample $(\mtx x, y)$ to a scalar, $E(\bm x, y)$, known as the \textit{energy}.
Energy values can be converted into a probability density $p(\bm x, y)$ in the form of \textit{Gibbs distribution} ($T$ is the temperature parameter):
\begin{equation}
    \label{eq:p(x,y)}
    p(\bm x, y) = \frac{\exp ( - E(\bm x, y) / T)}{Z},
\end{equation}
where $Z = \int_{\bm x} \sum_{i=1}^C \exp (- E(\bm x, i) / T)$ is the normalizing constant (also known as the partition function).
The probability density $p(\bm x)$ can then be computed as:
\begin{equation}
    \label{eq:p(x)}
    p(\bm x) = \sum_y p(\bm x, y) = \frac{\sum_y \exp ( - E(\bm x, y) / T)}{Z}.
\end{equation}
The normalization constant $Z$ is usually intractable to compute or reliably estimate over the input space. 
To address this, standard approaches in log-concave sampling is to take the negative logarithm of both sides in Eq.~\eqref{eq:p(x)}~\cite{chewi2023log}, giving:
\begin{equation}
    \label{eq:-logp(x)}
    - \log p(\bm x) = - \log \sum_y \exp\left( - E(\bm x, y) / T \right) + \log Z.
\end{equation}
The equation above indicates that omitting the term $Z$ does not affect OOD detection, as $\log Z$ is constant to each sample.
Consequently, we can define the \textit{Helmholtz} \textit{free} \textit{energy}
$E(\bm x)$ (here we reload the notation $E(\cdot)$ for convenience) as a surrogate of $- \log p(\bm x)$:
\begin{equation}
    E(\bm x) =  - T \log \sum_y \exp ( - E(\bm x, y) / T).
\end{equation}
Now, let us consider a \emph{fixed} classifier, $f_\theta: \mathbb{R}^D \mapsto \mathbb{R}^C$, with the model parameters $\theta$ that have already been trained on $\mathcal{D}_\text{ID}$.
A straightforward approach to construct $E(\cdot, \cdot)$ is to define the parameterized energy function as $E(\bm x, y; f_\theta) = - f_\theta^{(y)}(\bm x) / T$, where $f_\theta^{(i)}(\bm x)$ represents the $i$-th output of $f_\theta(\bm x)$.
The parameterized Helmholtz free energy, $E(\bm x; f_\theta)$, then becomes:
\begin{equation}
E(\mtx x; f_\theta) = - T \log \sum_{i=1}^{C}\exp(f_\theta^{(i)}(\mtx x)/T).
\label{eq:energy_function}
\end{equation}
Here, we simply set \( T = 1 \) for computational convenience in the following sections.
To this end, the score function is defined as $s_\theta(x) = - E(\bm x; f_\theta)$.
The OOD discriminator $D(\mtx x; \tau, f_\theta)$ is then given by:
\begin{equation}
    D(\mtx x; \tau, f_\theta)=
    \begin{cases} 
        0 & \text{if } - E(\mtx x; f_\theta) \leq \tau, \\
        1 & \text{if } - E(\mtx x; f_\theta) > \tau,
    \end{cases}
\label{eq:energy_detector}
\end{equation}
where $\tau$ is a threshold.
Samples with higher energy (lower score) values are considered as OOD inputs and vice versa.

As a closing remark, we note it is also feasible to further tune the classifier $f_\theta(\cdot)$ for better OOD detection performance~\cite{liu2020energy}.
We recall Eq.~\eqref{eq:energy_function} can serve as a surrogate for the negative log-likelihood of \( p(\bm{x}) \) for \emph{fixed $E(\bm x, y)$},
and following the spirit of MLE (maximum likelihood estimation) in score-based methods, users can intuitively turn to minimize the energy for better modeling the data.
We will revisit this tuning strategy in \Cref{sec: fine-tune} (see the paragraph ``Issues of $\mathcal{L}_{\text{energy}}$'').

\section{Methodology}
\label{sec:meth}

In this section, we first introduce the novel concept of peripheral-distribution samples in Section~\ref{sec: PD}.
Following that, we present a fresh understanding of Energy-Based Models (EBMs) through a new concept ``energy barrier'' proposed in Section~\ref{sec: barriers}.
This barrier effectively separates in-distribution and out-of-distribution samples.
Finally, in Section~\ref{sec: fine-tune}, we propose a straightforward tuning strategy that leverages PD samples to improve the robustness of existing classifiers.

\subsection{Peripheral-Distribution Samples}
\label{sec: PD}

To effectively dissect in- and out-of-distribution data, or even near-distribution data~\cite{tack2020csi}, this paper introduces a new concept, peripheral-distribution (PD) samples.
This concept arises from the lack of real outliers (OOD data) in training; due to this lack, augmented samples from a specially defined distribution are usually taken as proxies for such outliers.

It remains an open problem how those augmented samples are related to ID data. 
\cite{chen2020simple} recognized augmented samples as positive, while \cite{tack2020csi} discovered some augmentations (\eg, rotation) are beneficial when they are treated as negative.
Here, as shown in Figure~\ref{fig:clusted-features}, certain augmented samples are found to be peripheral to the ID data.
Motivated by the observation, we conceptually refer those augmented samples to peripheral-distribution data, 
an interpolation in the feature space to connect in-distribution and out-of-distribution samples.

From a practical standpoint, PD data can be generated by applying specific data augmentation transformations to ID samples. 
This view is inspired by the principle of contrastive learning \cite{chen2020simple}, which suggests that transformed data tends to remain close to the original data in the feature space while still exhibiting a shift in feature distribution.
As demonstrated in Figure~\ref{fig:clusted-features}, the representations of the augmented data indeed are located in the hypothesized interpolation zone between in- and out-of-distribution samples.

We give the formulation of PD data as follows. 
Consider a set $\mathcal{S}$ comprising different transformations, which can be either random or deterministic. 
For a given batch of ID samples $\mathcal{B} = \{\mtx x_i\}_{i=1}^{B}, \forall \mtx x_i \in \idimage$, we can generate peripheral-distribution samples $\mathcal{B}^+$ by augmenting $\mathcal{B}$ with transformations from the predefined collection $\mathcal{S}$ (which is fixed and thus we omit the dependence on it in the notation of $\mathcal{B}^+$). 
We thereby define the peripheral-distribution samples as 
\begin{equation}
    \mathcal{B}^+ = \bigcup_{S \in \mathcal{S}} \{\mathcal{B}_S\}~,~\text{where}~ \mathcal{B}_S:=\{S(\mtx x_i)\}^B_{i=1}~.
    \label{eq:peripheral-distribution}
\end{equation}
Later, for the given $\mathcal{X}_\text{ID}$, we denote the support of peripheral-distribution data as $\mathcal{X}_\text{PD}$, with $\mathcal{B}^+ \subset \mathcal{X}_\text{PD}$.

\begin{figure*}[!t]
  \centering
  \includegraphics[width=0.95\textwidth]{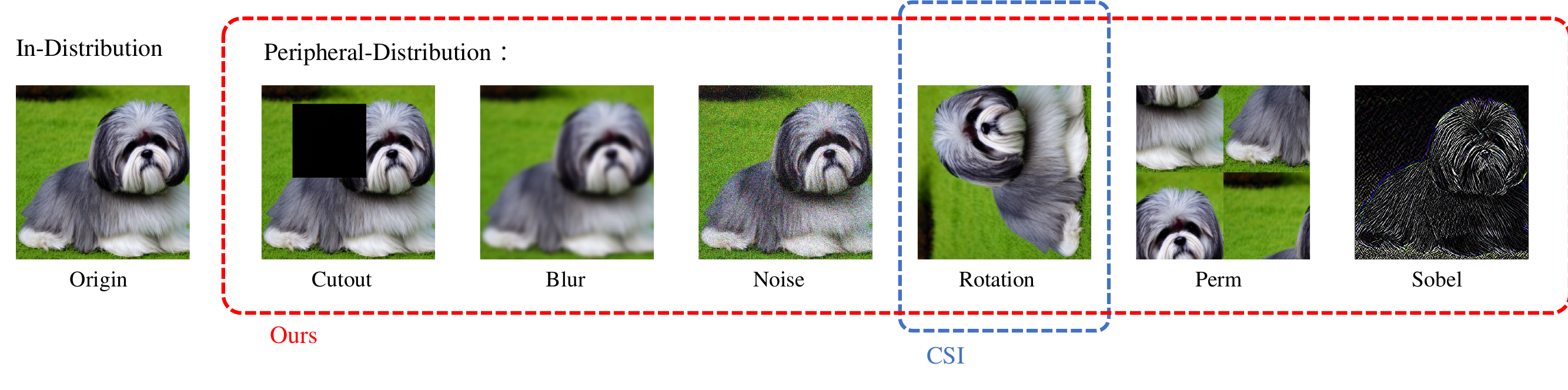}
  \vspace{-0.3cm}
  \caption{
  Visualization of the original image and the considered simple transformations. 
  The difference between our design and a baseline method CSI\protect\cite{tack2020csi} is also exhibited. 
  CSI must select different suitable transformations elaborately for each particular scenario, and specifically for CIFAR-10 CSI chooses rotation. However, our method utilizes all kinds of transformations.}
  \vspace{-0.2cm}
  \label{fig:simple_transformation}
\end{figure*}

\subsection{Energy Barrier Assumption}
\label{sec: barriers}

In this section, we demonstrate that each peripheral-distribution sample can provide an ``energy barrier'' (defined in \Cref{ass:periphery}), which benefits the classifier tuning.  
Similar to \cite{wang2022chaos, shen2022connect}, we make the following assumption regarding the representations of peripheral-distribution data and show how the energy barrier can indirectly expose the differences between in- and out-of-distribution samples.

Specifically, we consider a linear classifier (adopted in the most common softmax classifier): 
\begin{equation*}
f(\mtx x) \defeq \mtx C \mtx x,
\end{equation*}
where $\mtx C \in \mb R^D \times \mb R^C$ maps a sample $\mtx x$ from input space $\mathbb{R}^{D}$ to $C$ values, a.k.a.\ logits.
We denote the $i$-th row in $\mtx C$ as $\mtx c_i$, a length-$D$ vector indicating the corresponding class;
we sometimes call $\mtx c_i$ the ``class representation'' of class $i$.
Under the image classification setting, the linear classifier $f(\mtx x):\mathbb{R}^{D}\mapsto\mathbb{R}^{C}$ is usually the last layer of a neural network.
\begin{assumption}[Energy Barrier Assumption on Peripheral-Distribution Samples]
With the classifier $f(\cdot)$ and an out-of-distribution instance $\mtx x'$, we assume all the representations, including the class representations $\mtx c_i$'s, lie in a bounded domain with radius $B > 0$.
Moreover, for a random ID sample~$\mtx x$ and a certain probability level $\alpha \in (0, 1)$, 
there exists a certain augmented sample $\mtx x^+$ such that 
\begin{align}
E(\mtx x^+; f) - E(\mtx x; f) > B \|\mtx x' - \mtx x^+\| + \gamma_\alpha
\label{eqn:energy_barrier}
\end{align}
will hold with probability $1-\alpha$, where $\gamma_\alpha \geq 0$ is a constant.
\label{ass:periphery}
\end{assumption}
\begin{remark}
\normalfont
In addition to the usual compact domain assumption, we require there exists a large enough energy barrier ($\gamma_\alpha$ in Eq.~\eqref{eqn:energy_barrier}) between the original samples and one peripheral augmented sample $\mtx x^+$;
the high probability inequality also implies there is supposed to be one augmented sample $\mtx x^+$ closer to the out-of-distribution sample $\mtx x'$ than to most ID samples (otherwise the inequality will be invalid if $\|\mtx x' - \mtx x^+\|$ is overly large), 
which is heavily utilized in contrastive learning theory \cite{wang2022chaos, shen2022connect}\footnote{
For example, in scenarios where vehicle images are concerned, the certain augmentation cropping, which transforms an image of a vehicle to merely a tire, obviously moves its semantic boundary toward the out-of-distribution category.
}.
In this regard, peripheral augmented samples can help differentiate confusing OOD images close to ID samples.
\end{remark}

Through lifting the energy barrier $E(\mtx x^+; f) - E(\mtx x; f)$ as in \Cref{ass:periphery}, we can construct a gap between OOD and ID samples; the finding is formulated as follows. 
\begin{theorem}
When \Cref{ass:periphery} holds, we then have
$$
E(\mtx x'; f) - E(\mtx x; f) > \gamma_\alpha
$$
holds with probability $1-\alpha$.
The OOD sample $\mtx x'$ will be guaranteed to have higher energy than a random ID sample~$\mtx x$ with high probability.
\label{thm:periphery}
\end{theorem}
The proof is deferred to Appendix~\ref{sec:proof}.

\begin{remark}
\normalfont
To close this subsection, we remark the validity of \Cref{thm:periphery} heavily depends on the energy barrier assumption \Cref{ass:periphery}, which may not necessarily hold for the trained classifier $f(\cdot)$.
However, the theoretical result motivates our following empirical design, which aims to \emph{establish} an energy barrier between the original samples and the augmented ones.
\end{remark}

\subsection{Establishing the Energy Barrier via PD Samples}
\label{sec: fine-tune}

There are two requirements implied by \Cref{ass:periphery}, that 
\ding{182}~the augmented samples constitute a qualified semantics interpolation between ID and OOD samples and 
\ding{183}~there is an energy barrier between the aforementioned augmented samples and the ID data.
Our training strategy is accordingly composed of two parts, \ding{182}~choices of proper data augmentations, and \ding{183}~a carefully designed tuning objective.

\ding{182}~For the choices of data augmentations, we consider a flurry of regular transformations illustrated in Figure~\ref{fig:simple_transformation}: 
\begin{itemize}[leftmargin=*]
\item Geometric transformation: cutout\cite{cutout17arxiv}, permutation\cite{tack2020csi}, and rotation\cite{gidaris2018unsupervised}.
\item Appearance transformations: Gaussian noise, Gaussian blur, and Sobel filtering\cite{kanopoulos1988design}.
\end{itemize}
We note the transformations are beneficial considering they all contribute to a qualified interpolation when the classifier has already learned discriminative features from ID data, 
allowing augmented samples' representation to reside near the in-distribution samples' in the feature space, without completely overlapping with them.
The effects of each transformation are verified through the experiments in Section~\ref{subsec:simple transformations}.

\ding{183}~Considering the practical goal of OOD detection, which involves both classification and distinguishing OOD samples, the tuning objective consists of two components: the standard cross-entropy loss and an energy-based loss. 
Thus, the overall tuning objective is roughly: 
\begin{equation}
\min_{\theta}\quad \mathbb{E}_{(\mtx x, y), (\mtx x', y') \sim \id \\ } \; \mathcal{L}_{\text{CE}}(\mtx x, y) + \alpha \cdot \mathcal{L}_{\text{energy}}(\mtx x, \mtx x'),
\label{eq: training objective}
\end{equation}
where $\alpha$ is a loss scaling factor and we note $\mathcal{L}_{\text{energy}}(\mtx x, \mtx x')$ depends on the ID sample $\mtx x$ and another \iid~copy $\mtx x'$.
Previously, inspired by the energy-bounded learning objective proposed in \cite{liu2020energy}, which was originally designed for tuning with real outliers, we introduced a similar energy term for OEST \cite{wu2023oest}.
Particularly, the \emph{energy-bounded loss} $\mathcal{L}_{\text{energy}}$ in Eq.~\eqref{eq: training objective} for an ID input pair $(\mtx x_{\text{in}}, \mtx x_{\text{in}}')$ is:
\begin{equation}
\begin{aligned}
 \mathcal{L}_{\text {energy}}(\mtx x_{\text{in}}, \mtx x_{\text{in}}') &= \left( \max(0, E(\mtx x_{\text{in}}) - m_{\text{in}}) \right)^2 \\
 &\quad + \left( \max(0, m_{\text {per}} - E(\mtx x_{\text{per}})) \right)^2,
\end{aligned}
\label{eq: energy-bound-loss}
\end{equation}
where $m_{\text{in}}$ and $m_{\text{per}}$ are the margin hyper-parameters for the energy gap, 
and $\mtx x_{\text{per}}$ is the random PD sample augmented from $\mtx x_{\text{in}}'$.
Therefore, the energy loss $\mathcal{L}_{\text{energy}}$ penalizes the in-distribution samples whose energy values are higher than $m_{\text{in}}$ and the peripheral-distribution samples whose energy values are respectively lower than $m_{\text{per}}$.

\begin{figure}[t]
\vskip -0.0in
\makebox[0.01\textwidth]{} %
\includegraphics[width=0.43\textwidth]{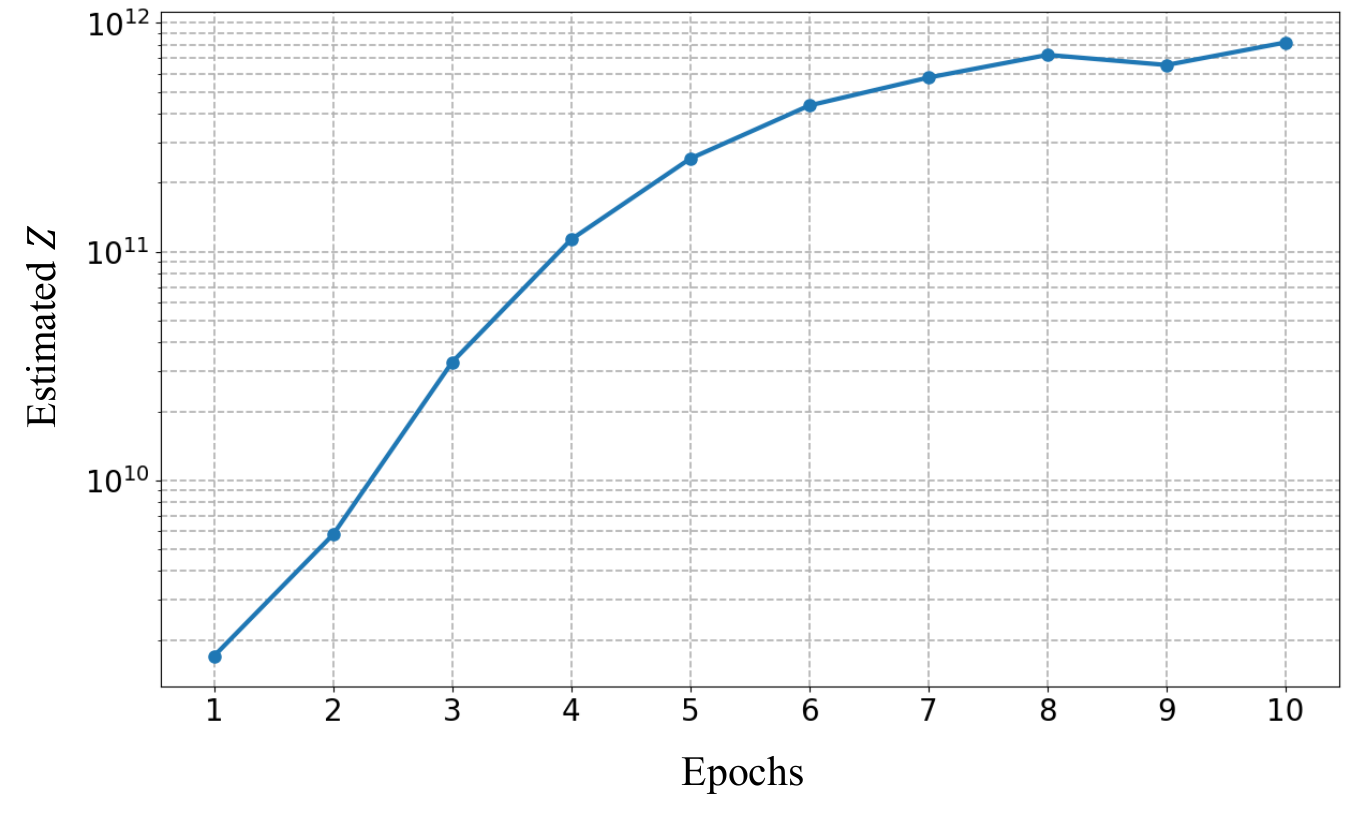}
\vskip -0.1in
\caption{Visualization of $Z$ during the tuning process of OEST on CIFAR-10. Here, $Z$ is computed as the empirical aggregation of all images used in~\cite{zhang2023openood}.}
\label{fig: Z}
\end{figure}

\textbf{Issues of $\mathcal{L}_{\text{energy}}$.}
Revisiting Eq.~\eqref{eq:-logp(x)}, we recall the energy \( E(\mtx x) \) represents \( -\log p(\mtx x) \), where we generally omit \( \log Z \) in practice due to the intractability of computing \( \log Z \). 
As discussed in Section~\ref{sec: EBM}, it is theoretically justifiable to ignore the normalization constant during inference using the fixed classifier $f(\cdot)$, but its omission becomes questionable in tuning $f(\cdot)$ with the energy-bounded loss $\mathcal{L}_{\text{energy}}$.

As a score-based method, the implicit goal of tuning is to maximize \( p(\mtx x_{\text{in}}) \) for ID samples;
although we once commented in Eq.~\eqref{eq:energy_function} that the energy function $E(\mtx x_{\text{in}})$ can serve as a surrogate for the negative log-likelihood of \( p(\mtx x_{\text{in}}) \) for \emph{fixed $E(\cdot)$},
we note in tuning, minimizing $E(\bm x)$ \textbf{does not equal} maximizing \( p(\bm{x}) \) considering $\log Z$ is changing as well.
As shown in Figure~\ref{fig: Z}, \( Z \) undergoes significant fluctuations during the tuning process, which echoes that ignoring \( Z \) in such scenarios could mislead the MLE objective. 

Therefore, we propose a new energy loss called \emph{energy-barrier loss}, which is given as:
\fontsize{9.6pt}{12pt}
\begin{equation}
\begin{aligned}
 \mathcal{L}_{\text {energy*}}(\mtx x_{\text{in}}, \mtx x_{\text{in}}') &= 
 \Big[\log \sigma \Big( 
 \big( E \left(\mtx x_{\text{per}} \right) - E\left (\mtx x_{\text{in}} \right) \big) / \beta \Big) \Big], 
\end{aligned}
\label{eq: energy-barrier-loss}
\end{equation}
\normalsize
where again $\mtx x_{\text{per}}$ is the random PD sample augmented from~$\mtx x_{\text{in}}'$, $\sigma(\cdot)$ is the sigmoid function, and $\beta$ is the hyper-parameter.
In Eq.~\eqref{eq: energy-barrier-loss}, the energy difference term $E(\mtx x_{\text{per}}) - E(\mtx x_{\text{in}})$ successfully removes the dependence on $\log Z$, making the formulation theoretically rigorous for the MLE spirit beneath the energy-based model \cite{lecun2006tutorial}.

Overall, $\mathcal{L}_{\text {energy*}}$ emphasizes the relative energy differences between in-distribution and peripheral-distribution samples, in line with the principles discussed in Section~\ref{sec: barriers}, \ie, establishing an energy barrier around the peripheral-distribution samples.
In the follow-up experiments, we will denote \emph{OEST} as the method with the energy loss $\mathcal{L}_{\text {energy}}$, and \emph{OEST*} as the method with the energy loss $\mathcal{L}_{\text {energy*}}$.

\begin{algorithm}[t!]
\caption{OEST* Tuning Algorithm for OOD Detection}
\begin{algorithmic}[1]
\REQUIRE Training data; original classifier $f_{\theta}(\cdot)$; transformation set $\mathcal{S}$; 
learning rate $\eta$; transformation ratio $\tau$; number of epochs $T$.
\ENSURE Tuned model $f_{\theta^*}$

\STATE Initialize the model parameter $\theta$
\FOR{each epoch $t = 1, 2, \dots, T$}
    \FOR{each mini-batch \((\mtx x, y)\) and \((\mtx x', y')\)}
        \STATE Compute the total loss: \\ 
        \quad \(\mathcal{L} \leftarrow \mathcal{L}_\text{CE}(\mtx x, y) + \alpha \cdot \mathcal{L}_\text{energy*}(\mtx x, \mtx x')\). \\
        \quad (In $\mathcal{L}_\text{energy*}$, a transformation from $\mathcal{S}$ is applied to \\
        \quad $\mtx x'$ with ratio $\tau$.)
        \STATE Perform backpropagation and update $\theta$ as: \\
        \quad \(\theta \leftarrow \theta - \eta \nabla_{\theta} \mathcal{L}\)
    \ENDFOR
\ENDFOR
\RETURN Tuned neural network parameters $\theta^*$.
\end{algorithmic}
\label{algo:oest}
\end{algorithm}

\section{Experiments}
\label{sec:experiments}

In this section, we conduct extensive experiments to validate the effectiveness of our method and compare its performance against existing approaches. Additionally, we perform comprehensive ablation studies to assess the impact of different components of the framework. It is important to note that in practical applications, the out-of-distribution (OOD) detection task is indeed twofold: a) accurately categorizing in-distribution samples, as in conventional classification tasks, and b) enabling a well-trained classifier to distinguish out-of-distribution samples during the inference phase correctly.

In Section~\ref{sec:setup}, we first introduce our experimental setup, including the datasets, evaluation metrics, and training details. Then, in Section~\ref{sec:results}, we present the main experimental results, showcasing the efficacy of our method across various datasets. We provide a thorough analysis of its performance, comparing it with other methods to demonstrate the robustness and reliability of our approach. Finally, in Section~\ref{sec:ablation}, we carry out exhaustive ablation studies to systematically examine the contribution of each component, offering deeper insights into their individual roles. These experiments are designed to comprehensively validate the effectiveness of our method, ensuring it performs well on both in-distribution classification and OOD detection.

\begin{table*}[htbp]
    \centering
    \caption{OOD detection performance (\%) on CIFAR-10. All the results are average values obtained from 3 random runs. The top-1 results are in \textbf{bold}, while the second- and third-best results are \underline{underlined}.}
    \label{tab: main_CIFAR-10}
    \setlength\tabcolsep{1pt}
    \resizebox{\textwidth}{!}{
        \begin{tabular}{l cc cc cc cc cc cc cc c}
            \toprule
            \multirow{3}{*}{\parbox{1.4cm}{Method}}
            &\multicolumn{2}{c}{CIFAR-100}		
            &\multicolumn{2}{c}{Tin}		
            &\multicolumn{2}{c}{MNIST}		
            &\multicolumn{2}{c}{SVHN}		
            &\multicolumn{2}{c}{Textures}		
            &\multicolumn{2}{c}{Places365}		
            &\multicolumn{2}{c}{Average}		
            &\multirow{3}{*}{ID ACC $\uparrow$}
            \\
            \cmidrule(lr){2-3} \cmidrule(lr){4-5}\cmidrule(lr){6-7} \cmidrule(lr){8-9} \cmidrule(lr){10-11} \cmidrule(lr){12-13} \cmidrule(lr){14-15}
            &AUROC$\uparrow$	&FPR95$\downarrow$	&AUROC$\uparrow$	&FPR95$\downarrow$	&AUROC$\uparrow$	&FPR95$\downarrow$	&AUROC$\uparrow$	&FPR95$\downarrow$	&AUROC$\uparrow$	&FPR95$\downarrow$	&AUROC$\uparrow$	&FPR95$\downarrow$	&AUROC$\uparrow$	&FPR95$\downarrow$
            \\
            \midrule
            \multicolumn{16}{c}{\textbf{Post-Hoc Inference Methods}} \\

            ASH\cite{djurisic2023ash} & 74.11 & 87.31 & 76.44 & 86.25 & 83.16 & 70.00 & 73.46 & 83.64 & 77.45 & 84.59 & 79.89 & 77.89 & 77.42  & 81.61  & \underline{95.06} \\
            SHE\cite{zhang2023she} & 80.31 & 81.00 & 82.76 & 78.30 & 90.43 & 42.22 & 86.38 & 62.74 & 81.57 & 84.60 & 82.89 & 76.36 & 84.06  & 70.87  & \underline{95.06} \\
            ODIN\cite{liang2018odin} & 82.18 & 77.00 & 83.55 & 75.38 & 95.24 & 23.83 & 84.58 & 68.61 & 86.94 & 67.70 & 85.07 & 70.36 & 86.26  & 63.81  & \underline{95.06} \\ 
            MSP\cite{hendrycks2017abaseline} & 87.19 & 53.08 & 88.87 & 43.27 & 92.63 & 23.64 & 91.46 & 25.82 & 89.89 & 34.96 & 88.92 & 42.47 & 89.83  & 37.21  & \underline{95.06} \\ 
            MLS\cite{hendrycks2022scaling} & 86.31 & 66.59 & 88.72 & 56.06 & 94.15 & 25.06 & 91.69 & 35.09 & 89.41 & 51.73 & 89.14 & 54.84 & 89.90  & 48.23  & \underline{95.06} \\ 
            EBO\cite{liu2020energy} & 86.36 & 66.60 & 88.80 & 56.08 & 94.32 & 24.99 & 91.79 & 35.12 & 89.47 & 51.82 & 89.25 & 54.85 & 90.00  & 48.24  & \underline{95.06} \\ 
            TempScale\cite{guo2017calibration} & 87.17 & 55.81 & 89.00 & 46.11 & 93.11 & 23.53 & 91.66 & 26.97 & 90.01 & 38.16 & 89.11 & 45.27 & 90.01  & 39.31  & \underline{95.06} \\
            GEN\cite{liu2023gen} & 87.21 & 58.75 & 89.20 & 48.59 & 93.83 & 23.00 & 91.97 & 28.14 & 90.14 & 40.74 & 89.46 & 47.03 & 90.30  & 41.04  & \underline{95.06} \\
            KNN\cite{sun2022out} & 89.73 & 37.64 & 91.56 & 30.37 & 94.26 & 20.05 & 92.67 & 22.60 & 93.16 & 24.06 & 91.77 & 30.38 & 92.19  & 27.52  & \underline{95.06} \\

            \midrule
            \multicolumn{16}{c}{\textbf{Training methods from scratch}} \\

            MOS\cite{huang2021mos} & 70.57 & 79.38 & 72.34 & 78.05 & 74.81 & 65.95 & 73.66 & 57.79 & 70.35 & 76.78 & 86.81 & 51.09 & 74.76  & 68.17  & 94.83 \\
            ARPL\cite{chen2021adversarial} & 86.76 & 43.38 & 88.12 & 37.28 & 92.62 & 21.49 & 87.69 & 35.68 & 88.57 & 35.19 & 88.57 & 37.21 & 88.72  & 35.04  & 93.66 \\
            VOS\cite{du2022vos} & 86.57 & 61.57 & 88.84 & 52.49 & 91.56 & 35.92 & 92.18 & 31.50 & 89.68 & 46.53 & 89.90 & 47.78 & 89.79  & 45.97  & 94.31 \\ 
            CSI\cite{tack2020csi} & 88.16 & 37.57 & 90.87 & 29.74 & 92.55 & 24.41 & 95.18 & 17.56 & 90.71 & 28.95 & 89.56 & 34.76 & 91.17  & 28.83  & 91.16 \\ 
            ConfBranch\cite{devries2018learning} & 88.91 & 34.44 & 90.77 & 28.11 & 94.49 & 15.79 & 95.42 & 14.06 & 91.10 & 27.24 & 90.39 & 28.85 & 91.85  & 24.75  & 94.88 \\ 
            NPOS\cite{tao2022npos} & 88.57 & 35.71 & 90.99 & 29.57 & 92.64 & 22.96 & 98.88 & 6.41 & 94.44 & 20.80 & 90.32 & 32.19 & 92.64  & 24.61  & / \\ 
            CIDER\cite{ming2022exploit} & 89.47 & 35.60 & 91.94 & 28.61 & 93.30 & 24.76 & 98.06 & 8.04 & 93.71 & 25.05 & 93.77 & 25.03 & 93.38  & 24.52  & / \\ 
            G-ODIN\cite{hsu2020generalized} & 88.14 & 48.86 & 90.09 & 42.21 & 98.95 & \underline{4.53} & 97.76 & 10.72 & 95.02 & 27.27 & 90.31 & 43.30 & 93.38  & 29.48  & 94.70  \\
            LogitNorm\cite{wei2022mitigating} & 90.95 & 34.37 & 93.70 & 24.30 & \underline{99.14} & \underline{3.93} & 98.25 & 8.33 & 94.77 & 21.94 & \underline{94.79} & \underline{21.04} & 95.27  & 18.99  & 94.30 \\ 
            RotPred\cite{hendrycks2019using} & \underline{91.19} & \underline{34.24} & \underline{94.17} & \underline{22.04} & 97.52 & 9.24 & \underline{98.89} & \textbf{3.20} & \underline{97.30} & \underline{9.87} & 92.76 & 26.61 & \underline{95.31}  & \underline{17.53}  & \textbf{95.35} \\ 
            \midrule
            \multicolumn{16}{c}{\textbf{Tuning method}} \\
            OEST (Ours) & \underline{91.27} & \underline{33.46} & \underline{94.62} & \underline{21.57} & \textbf{99.65} & \textbf{1.83}  & \underline{99.10} & \underline{5.07} & \underline{97.87} & \underline{11.03} & \underline{94.80} & \underline{20.70} & \underline{96.22} & \underline{15.61} & \underline{95.00} \\
            OEST* (Ours) & \textbf{91.47} & \textbf{32.60} & \textbf{94.81} & \textbf{21.02} & \underline{98.98} & 4.97 & \textbf{99.28} & \underline{3.75} & \textbf{98.18} & \textbf{9.52} & \textbf{95.10} & \textbf{20.13} & \textbf{96.30} & \textbf{15.33} & 94.97 \\

            \bottomrule
        \end{tabular}
    }
\end{table*}

\subsection{Setup}
\label{sec:setup}

\subsubsection{Datasets}
We follow a common setup in the out-of-distribution detection field and mainly report results using two widely used datasets, CIFAR-10 and CIFAR-100 \cite{krizhevsky2009learning}. In terms of OOD datasets, we primarily adhere to the practices outlined in OpenOOD \cite{conf/nips/YangWZZDPWCLSDZ22:openood}.
When CIFAR-10 \cite{krizhevsky2009learning} is used as the in-distribution dataset, CIFAR-100 \cite{krizhevsky2009learning} and Tiny ImageNet (Tin) \cite{le2015tiny} are used as near OOD datasets, while MNIST \cite{lecun1998gradient}, SVHN \cite{svhn}, Textures \cite{cimpoi2014describing}, and Place365 \cite{zhou2017places} are employed as far OOD datasets.
Similarly, for CIFAR-100 \cite{krizhevsky2009learning} as the in-distribution dataset, we adopt CIFAR-10 \cite{krizhevsky2009learning} and Tiny ImageNet (Tin) \cite{le2015tiny} as near OOD datasets, and MNIST \cite{lecun1998gradient}, SVHN \cite{svhn}, Textures \cite{cimpoi2014describing}, and Place365 \cite{zhou2017places} as far OOD datasets.
For a detailed description of all datasets used, please refer to Appendix~\ref{appendix: dataset}.

Additionally, for the ablation studies, we use KMNIST \cite{clanuwat2018deep} as the in-distribution dataset, with CIFAR-10 \cite{krizhevsky2009learning} and EMNIST \cite{cohen2017emnist} serving as real outliers for tuning, and MNIST \cite{lecun1998gradient} for OOD evaluation. Moreover, we conduct further experiments with MNIST and SVHN as in-distribution datasets, the details of which can be found in Appendix~\ref{appendix: mnist&svhn}.

\subsubsection{Evaluation Metrics}
As mentioned earlier, OOD detection in practical applications has two key objectives: accurate classification of in-distribution samples and reliable detection of out-of-distribution (OOD) samples. To rigorously evaluate the effectiveness of our methods, we assess the results using three widely recognized metrics.
The first two metrics focus on OOD detection performance. The first is the Area Under the Receiver Operating Characteristic Curve (AUROC) \cite{davis2006relationship}, which provides a probabilistic measure of the likelihood that a positive sample receives a higher discriminative score than a negative one \cite{fawcett2006introduction}. AUROC serves as a comprehensive indicator of the model’s ability to differentiate between in-distribution and OOD samples. The second metric is the False Positive Rate at 95\% True Positive Rate (FPR95) \cite{liang2018odin}, which measures how often negative samples are mistakenly classified as positive. FPR95 is particularly useful for assessing the model’s reliability in scenarios where precise classification of positive samples is critical.
The third metric, in-distribution testing accuracy (ID-ACC), reflects the model’s performance on the original classification task. This metric ensures that the model maintains strong classification capabilities, which is crucial along with OOD detection.
By incorporating these three metrics, we emphasize the primary goal of OOD detection: achieving a balanced performance across both the classification task and the detection of OOD samples.

\begin{table*}[htbp]
    \centering
    \caption{OOD detection performance (\%) on CIFAR-100. All the results are average values obtained from 3 random runs. The top-1 results are in \textbf{bold}, while the second- and third-best results are \underline{underlined}.}
    \label{tab: main_CIFAR-100}
    \setlength\tabcolsep{1pt}
    \adjustbox{width=1\textwidth}{
        \begin{tabular}{l cc cc cc cc cc cc cc c}
            \toprule
            \multirow{3}{*}{\parbox{1.6cm}{Method}}
            &\multicolumn{2}{c}{CIFAR-10}		
            &\multicolumn{2}{c}{Tin}		
            &\multicolumn{2}{c}{MNIST}		
            &\multicolumn{2}{c}{SVHN}		
            &\multicolumn{2}{c}{Textures}		
            &\multicolumn{2}{c}{Places365}		
            &\multicolumn{2}{c}{Average}		
            &\multirow{3}{*}{ID ACC $\uparrow$}
            \\
            \cmidrule(lr){2-3} \cmidrule(lr){4-5}\cmidrule(lr){6-7} \cmidrule(lr){8-9} \cmidrule(lr){10-11} \cmidrule(lr){12-13} \cmidrule(lr){14-15}
            &AUROC$\uparrow$	&FPR95$\downarrow$	&AUROC$\uparrow$	&FPR95$\downarrow$	&AUROC$\uparrow$	&FPR95$\downarrow$	&AUROC$\uparrow$	&FPR95$\downarrow$	&AUROC$\uparrow$	&FPR95$\downarrow$	&AUROC$\uparrow$	&FPR95$\downarrow$	&AUROC$\uparrow$	&FPR95$\downarrow$
            \\
            \midrule
            \multicolumn{16}{c}{\textbf{Post-Hoc Inference Methods}} \\
            
            OpenMax\cite{bendale2016towards} & 74.38 & 60.17 & 78.44 & 52.99 & 76.01 & 53.82 & 82.07 & 53.20 & 80.56 & 56.12 & 79.29 & \underline{54.85} & 78.46  & 55.19  & \underline{77.25} \\ 
            MSP\cite{hendrycks2017abaseline} & 78.47 & 58.91 & 82.07 & 50.70 & 76.08 & 57.23 & 78.42 & 59.07 & 77.32 & 61.88 & 79.22 & 56.62 & 78.60  & 57.40  & \underline{77.25} \\
            TempScale\cite{guo2017calibration} & 79.02 & \textbf{58.72} & 82.79 & 50.26 & 77.27 & 56.05 & 79.79 & 57.71 & 78.11 & 61.56 & 79.80 & 56.46 & 79.46  & 56.79  & \underline{77.25} \\ 
            ODIN\cite{liang2018odin} & 78.18 & 60.64 & 81.63 & 55.19 & 83.79 & 45.94 & 74.54 & 67.41 & 79.33 & 62.37 & 79.45 & 59.71 & 79.49  & 58.54  & \underline{77.25} \\ 
            MLS\cite{hendrycks2022scaling} & \underline{79.21} & \underline{59.11} & 82.90 & 51.83 & 78.91 & 52.95 & 81.65 & 53.90 & 78.39 & 62.39 & 79.75 & 57.68 & 80.14  & 56.31  & \underline{77.25} \\ 
            EBO\cite{liu2020energy} & 79.05 & 59.21 & 82.76 & 52.03 & 79.18 & 52.62 & 82.03 & 53.62 & 78.35 & 62.35 & 79.52 & 57.75 & 80.15  & 56.26  &  \underline{77.25} \\ 
            GEN\cite{liu2023gen} & \textbf{79.38} & \underline{58.87} & \underline{83.25} & 49.98 & 78.29 & 53.92 & 81.41 & 55.45 & 78.74 & 61.23 & 80.28 & 56.25 & 80.23  & 55.95  & \underline{77.25} \\ 
            ReAct\cite{sun2021react} & 78.65 & 61.30 & 82.88 & 51.47 & 78.37 & 56.04 & 83.01 & 50.41 & 80.15 & 55.04 & 80.03 & \underline{55.30} & 80.52  & 54.93  &  \underline{77.25} \\ 
            KNN\cite{sun2022out} & 77.02 & 72.80 & \underline{83.34} & \underline{49.65} & 82.36 & 48.58 & 84.15 & 51.75 & 83.66 & 53.56 & 79.43 & 60.70 & 81.66  & 56.17  & \underline{77.25} \\ 
            RMDS\cite{ren2021simple} & 77.75 & 61.37 & 82.55 & \underline{49.56} & 79.74 & 52.05 & 84.89 & 51.65 & 83.65 & 53.99 & \textbf{83.40} & \textbf{53.57} & 82.00  & 53.70  & \underline{77.25} \\
            \midrule
            
            \multicolumn{16}{c}{\textbf{Training methods from scratch}} \\
            
            CSI\cite{tack2020csi} & 69.50  & 72.62 & 73.40  & 67.90  & 51.79 & 80.54 & 80.24 & 67.21 & 62.22 & 90.51  & 70.99 & 69.41 & 68.02  & 74.70  & 61.60  \\ 
            ConfBranch\cite{devries2018learning} & 68.80  & 74.56 & 74.41 & 65.86  & 74.29 & 55.95 & 65.51 & 76.01 & 65.39 & 85.43 & 70.42 & 69.90 & 69.80  & 71.29  & 76.59 \\ 
            ARPL\cite{chen2021adversarial} & 73.38 & 64.84 & 76.50  & 58.27  & 73.77 & 59.12 & 76.45 & 59.76 & 69.93 & 71.66  & 74.62 & 62.01 & 74.11  & 62.61  & 70.70  \\ 
            CIDER\cite{ming2022exploit} & 67.55 & 82.71 & 78.65 & 61.33  & 68.14 & 75.32 & \underline{97.17} & \underline{17.82} & 82.21 & 54.43 & 74.43 & 69.30 & 78.03  & 60.15  & / \\ 
            MOS\cite{huang2021mos} & 78.54 & 60.60 & 82.26 & 51.49 & 80.68 & 52.70 & 81.59 & 56.33 & 79.92 & 61.24 & 78.50 & 58.86 & 80.25  & 56.87  & \underline{76.98} \\
            LogitNorm\cite{wei2022mitigating} & 74.57 & 73.88 & 82.37 & 51.89 & 90.69 & 34.12 & 82.80  & 47.52 & 72.37 & 77.38  & 80.25 & 55.44 & 80.51  & 56.71  & 76.34 \\ 
            NPOS\cite{tao2022npos} & 75.37 & 72.50 & 81.32 & 54.21 & 73.26 & 66.98 & 92.43 & 30.67 & \underline{85.55} & 47.39 & 77.92 & 59.47 & 80.98  & 55.20  & / \\ 
            VOS\cite{du2022vos} & \underline{79.14} & 59.23 & 82.73 & 51.89 & 82.29 & 48.56 & 84.23 & 47.23 & 78.41 & 62.55 & \underline{80.34} & 56.44 & 81.19  & 54.32  & 77.20 \\
            G-ODIN\cite{hsu2020generalized} & 73.04 & 78.82 & 81.26 & 56.34 & \underline{91.15} & \underline{27.19} & 83.74 & 42.68 & \underline{89.62} & \textbf{35.83} & 78.17 & 65.03 & 82.83  & 50.98  & 74.46 \\ 
            RotPred\cite{hendrycks2019using} & 71.11 & 72.00 & 81.75 & 53.17 & \underline{93.10}  & \underline{22.77} & \underline{95.39} & \underline{15.64} & 88.16 & 40.03 & 76.95 & 59.56 & \underline{84.41}  & \underline{43.86}  & 76.03 \\ 
            
            \midrule
            \multicolumn{16}{c}{\textbf{Tuning method}} \\
            OEST (Ours) & 77.43 & 63.35 & 82.71 & 52.20 & 90.75 & 33.24 & 88.54 & 36.18 & 81.44 & 55.12 & 79.03 & 58.58 & \underline{83.32} & \underline{49.78} & 76.87 \\
            OEST* (Ours) & 75.22 & 69.94 & \textbf{85.75} & \textbf{47.29} & \textbf{95.68} & \textbf{19.39} & \textbf{98.56} & \textbf{8.04} & \textbf{90.32} & \underline{36.30} & \underline{82.63} & \underline{54.51} & \textbf{88.03} & \textbf{39.25} & \textbf{77.63} \\
            \bottomrule
        \end{tabular}
    }
\end{table*}

\subsubsection{Training Details} 
We primarily report results obtained by using ResNet\cite{he2016deep}.
For the CIFAR-10 and CIFAR-100 datasets, we use the trained ResNet-18 model provided by \cite{zhang2023openood}, which is trained with SGD optimizer using a learning rate of 0.1, momentum of 0.9, and weight decay of $5 \times 10^{-4}$ for 100 epochs.
We further tune the trained model for an additional 10 epochs, again using the SGD optimizer. 
For CIFAR-10, we apply six augmentations—cutout, blur, noise, rotation, permutation, and sobel—comprising a mix of geometric and appearance transformations to generate peripheral-distribution data. 
The ratio of in-distribution to peripheral-distribution data is set to 1:1, with a batch size of 128 for the in-distribution data.
For CIFAR-100, in addition to these six augmentations, we also apply a seventh augmentation, \emph{RandAugment}, a more diverse transformation technique from the PyTorch library. 
This inclusion provides a wider range of transformations, further enriching the peripheral-distribution data for CIFAR-100. 
In this case, the ratio of in-distribution to peripheral-distribution data is set to 1:2, with a batch size of 128 for the in-distribution data.
For OEST, we apply a weight of \( \alpha = 0.01 \) for \( L_{\text{energy}} \), with the margin values \( m_{\text{in}} \) and \( m_{\text{pre}} \) set to \(-25\) and \(-7\), respectively, following the setup in \cite{liu2020energy}. 
The learning rate follows a cosine annealing schedule, starting at \( 1 \times 10^{-4} \) and gradually decaying to \( 1 \times 10^{-8} \) throughout tuning.
For OEST*, we use weights of \( \alpha = 0.2 \) and \( \beta = 10 \) for \( L_{\text{energy*}} \). 
The learning rate also follows a cosine annealing schedule, starting at \( 1 \times 10^{-3} \) and gradually decaying to \( 1 \times 10^{-7} \) during tuning.
Specifically, for our methods, we further tune the three trained models provided by \cite{zhang2023openood} using a random seed of 1 to ensure fairness and reproducibility in our comparisons.
Additional ablation studies on other hyper-parameters are provided in the subsequent subsection.

\subsection{Main Results}
\label{sec:results}

To thoroughly assess the performance of the proposed method across various scenarios, we conduct a rigorous comparison not only against baseline models but also against several algorithms that have gained recognition in recent years. This comprehensive evaluation allows us to gain a deeper understanding of the proposed method.
We select methods from the best average detection performance under a unified evaluation benchmark \cite{conf/nips/YangWZZDPWCLSDZ22:openood}. The results are presented in Table~\ref{tab: main_CIFAR-10} and Table~\ref{tab: main_CIFAR-100}.
It should be noted that certain approaches, such as PixMix \cite{hendrycks2022pixmix}, are not directly comparable to ours, as they leverage a manually curated set of auxiliary images for mixing.

\begin{table*}[!tbp]
    \centering
    \caption{The comparison with different single data augmentation by each simple transformations about AUROC (\%) when CIFAR-10 is the given in-distribution. The first model is the trained classifier, the last one is the model further tuned with the composition of all considered data augmentations. The top-1 results are in \textbf{bold}, while the second-best results are \underline{underlined}. }
    \label{tab: data augmentation}
    \setlength\tabcolsep{1pt}
    \resizebox{\textwidth}{!}{
    \begin{tabular}{c@{\hspace{8pt}} cc cc cc cc cc cc cc}
        \toprule
           \multirow{2}{*}{\makecell[c]{Simple\\Transform}}
        &\multicolumn{2}{c}{CIFAR-100}		
        &\multicolumn{2}{c}{Tin}		
        &\multicolumn{2}{c}{MNIST}		
        &\multicolumn{2}{c}{SVHN}		
        &\multicolumn{2}{c}{Textures}		
        &\multicolumn{2}{c}{Places365}			
        &\multicolumn{2}{c}{Average}	\\
        \cmidrule(lr){2-3} \cmidrule(lr){4-5}\cmidrule(lr){6-7} \cmidrule(lr){8-9} \cmidrule(lr){10-11} \cmidrule(lr){12-13} \cmidrule(lr){14-15}
         &AUROC$\uparrow$	&FPR95$\downarrow$	&AUROC$\uparrow$	&FPR95$\downarrow$	&AUROC$\uparrow$	&FPR95$\downarrow$	&AUROC$\uparrow$	&FPR95$\downarrow$	&AUROC$\uparrow$	&FPR95$\downarrow$	&AUROC$\uparrow$	&FPR95$\downarrow$	&AUROC$\uparrow$	&FPR95$\downarrow$
        \\
        \midrule
        EBO \cite{liu2020energy} & 86.36 & 66.60 & 88.80 & 56.08 & 94.32 & 24.99 & 91.79 & 35.12 & 89.47 & 51.82 & 89.25 & 54.85 & 90.00  & 48.24 \\
        Sobel & 90.68 & 43.83 & 90.23 & 41.06 & \underline{99.01} & \underline{4.38} & 93.73 & 38.43 & 91.34 & 39.34 & 92.02 & 36.04 & 93.00  & 33.85 \\ 
        Blur & 90.04 & 40.88 & 92.20 & 30.71 & 97.17 & 12.49 & \textbf{99.77} & \textbf{0.61} & 92.17 & 34.00 & 91.97 & 31.84 & 93.89 & 25.09 \\
        Noise & 83.68 & 61.84 & 89.97 & 44.06 & 85.50 & 50.39 & 87.84 & 38.17 & 93.06 & 33.16 & 88.13 & 50.68 & 88.03 & 46.38 \\
        Cutout & 89.49 & 43.21 & 92.77 & 30.06 & \textbf{99.94} & \textbf{0.14} & 93.29 & 19.33 & 93.76 & 25.92 & 94.77 & 21.09 & 94.00 & 23.29 \\
        Perm & 91.10 & 34.13 & \underline{94.66} & \underline{22.89} & 96.93 & 12.71 & 96.87 & 12.31 & 97.60 & 11.42 & \textbf{96.45} & \textbf{16.04} & 94.00 & 23.29 \\ 
        Rotation & \textbf{91.49} & \underline{32.79} & 94.41 & 22.97 & 97.45 & 10.41 & 99.17 & 4.36 & \underline{97.70} & \underline{10.34} & 93.76 & 26.00 & \underline{95.83} & \underline{16.18} \\
        OEST* & \underline{91.47} & \textbf{32.60} & \textbf{94.81} & \textbf{21.02} & 98.98 & 4.97 & \underline{99.28} & \underline{3.75} & \textbf{98.18} & \textbf{9.52} & \underline{95.10} & \underline{20.13} & \textbf{96.30} & \textbf{15.33} \\

        \bottomrule
    \end{tabular}
    }
\end{table*}

\begin{table}[tb]
    \centering
    \caption{AUROC (\%) performance comparison between peripheral-distribution samples and real outliers. The top-1 results are in \textbf{bold}, while the second-best results are \underline{underlined}.}
    \label{tab: problem of external dataset}
    \begin{tabular}{cccc}
    \toprule
    $\id$ & \makecell[c]{External\\dataset?} & $\oodimage$ or $\mathcal{X}_\text{PD}$  & MNIST \\
    \midrule
    \multirow{4}{6em}{KMNIST} & \XSolidBrush & EBO \cite{liu2020energy} & 89.6 \\
    & \CheckmarkBold & CIFAR-10 & 91.4 \\
    & \CheckmarkBold & EMNIST & \textbf{99.8} \\
    & \XSolidBrush & OEST* (Ours) & \underline{98.6} \\
    \bottomrule
    \end{tabular}
\end{table}

\subsubsection{CIFAR-10 as ID} 

Table~\ref{tab: main_CIFAR-10} presents the OOD detection performance for CIFAR-10 as the in-distribution dataset, evaluated across six out-of-distribution test datasets. 
Our methods consistently achieve the top-3 AUROC on all six OOD datasets. 
Notably, the average AUROC of OEST* reaches 96.30\%, surpassing RotPred by 0.99\%. 
Additionally, our methods also achieve the lowest FPR95 on almost OOD datasets so that OEST* achieves the lowest average FPR95 (15.33\%).
The outstanding average performance on both AUROC and FPR95 demonstrates our methods' superior ability to maintain detection accuracy while minimizing false positives.
Moreover, with a more theoretically solid loss function, OEST* demonstrates a more balanced and superior performance compared to OEST.

Since post-hoc inference methods do not alter the model’s structure and only adjust the computation of the final scoring function, the accuracy on in-distribution samples remains unchanged.
In contrast, methods that involve training from scratch or continual tuning typically result in a decrease in ID-ACC. This reduction arises from the inherent trade-off between optimizing for out-of-distribution (OOD) detection and maintaining in-distribution accuracy. During training, Empirical Risk Minimization (ERM) focuses on minimizing classification error for in-distribution samples, while OOD detection often requires a different optimization direction. This divergence can create a mismatch between the objectives of enhancing OOD detection performance and preserving the accuracy of the original classification task, leading to a potential drop in ID-ACC. However, our tuning strategy resulted in only a 0.06\% decrease in classification accuracy for OEST and a 0.09\% decrease for OEST*—differences so small they can be considered negligible, particularly in the context of CIFAR-10 tasks.

\subsubsection{CIFAR-100 as ID}

As shown in Table~\ref{tab: main_CIFAR-100}, our approach significantly surpasses other baselines, achieving the highest average AUROC of 88.03\% (+3.62\%) and the lowest average FPR95 of 39.25\% (-3.61\%). Notably, OEST* exhibits exceptional AUROC performance on Tiny ImageNet (85.75\%), MNIST (95.68\%), SVHN (98.56\%), and Texture (90.32\%), underscoring its effectiveness across both near-OOD and far-OOD datasets.
Furthermore, our tuning process not only enhances out-of-distribution detection but also improves the original model’s classification accuracy, reflected in a higher ID-ACC score of 77.63\%. This improvement suggests that our method promotes more robust feature representations, benefiting both in-distribution classification and OOD detection.
In this context, OEST* outperforms OEST on five out of six out-of-distribution datasets, highlighting the advantages of the energy-barrier loss \( \mathcal{L}_{\text{energy*}} \). With a more rigorous theoretical foundation than the energy-bounded loss, the energy-barrier loss enables OEST* to achieve superior results, further emphasizing the necessity of accounting for \( \log Z \) during training, as it cannot be simply disregarded.
However, we acknowledge that our methods are less effective on CIFAR-10. This phenomenon is discussed in detail in Section~\ref{sec: assumption}.

\subsection{Experimental Analysis}
\label{sec:ablation}
Building on the experimental results discussed above, we conducted a series of comprehensive experiments to explore various factors of our method. More discussion can be found in Appendix~\ref{appendix: experimental analysis}.

\subsubsection{The Ablation Study for Simple Transformations}
\label{subsec:simple transformations}
As illustrated in Table~\ref{tab: data augmentation}, we have observed that almost all transformations individually improve the classifier's performance. Consequently, we decided to combine all the transformations together to evaluate their collective impact. The results confirm that the mixed version effectively enhances the performance of the trained classifier, demonstrating better comprehensive OOD detection performance in terms of both AUROC and FPR95 metrics.
Although the mixed version may not consistently outperform individual transformations in certain benchmarks, it still yields the highest overall improvement for OOD detection.

\subsubsection{Limitation of Real Outliers} 

We evaluate classifiers on KMNIST as the in-distribution dataset, with MNIST as the OOD dataset, as shown in Table~\ref{tab: problem of external dataset}. 
All models are based on the LeNet backbone; the first model is a trained classifier, and the other three are further tuned based on the first one, utilizing either external datasets or samples augmented with simple transformations (cutout, blur, noise, and permutation).
The results indicate that using CIFAR-10 as auxiliary data provides only a marginal improvement in OOD detection performance (AUROC 91.4\%). 
This is likely because KMNIST \cite{clanuwat2018deep}, EMNIST \cite{cohen2017emnist}, and MNIST \cite{lecun1998gradient} are all grayscale datasets containing handwritten images, whereas CIFAR-10 \cite{krizhevsky2009learning} comprises colorful images of various natural scenes, making it stylistically distinct from MNIST. 
By contrast, using EMNIST as auxiliary data, which has a closer resemblance in style and content to MNIST, significantly boosts performance, achieving a top AUROC of 99.8\%. 
This finding underscores the importance of choosing an appropriate auxiliary dataset in outlier-based methods and reveals how dataset selection can limit the applicability of these approaches.
Thus, for OOD detection, it is generally assumed that auxiliary OOD training data is not accessible.
Additionally, our proposed peripheral-distribution samples demonstrate performance comparable to real outliers (AUROC 98.6\%), further validating the effectiveness of our approach in the absence of specific auxiliary datasets.

\begin{table*}[tb]
    \centering
    \caption{The comparison of our method with different backbones\cite{zagoruyko2016wide} about AUROC (\%) when CIFAR-10 is the given in-distribution. \textbf{Bold} denotes the best results.}
    \label{tab:backbone}
    \setlength\tabcolsep{1pt}
    \resizebox{0.99\textwidth}{!}{
    \begin{tabular}{c@{\hspace{8pt}} c @{\hspace{8pt}}cc cc cc cc cc cc cc}
        \toprule
        \multirow{2}{*}{\makecell[c]{Backbone}}
        & \multirow{2}{*}{\makecell[c]{Further Tuned?}}
        &\multicolumn{2}{c}{CIFAR-100}		
        &\multicolumn{2}{c}{Tin}		
        &\multicolumn{2}{c}{MNIST}		
        &\multicolumn{2}{c}{SVHN}		
        &\multicolumn{2}{c}{Textures}		
        &\multicolumn{2}{c}{Places365}		
        &\multicolumn{2}{c}{Average}	\\
        \cmidrule(lr){3-4} \cmidrule(lr){5-6}\cmidrule(lr){7-8} \cmidrule(lr){9-10} \cmidrule(lr){11-12} \cmidrule(lr){13-14} \cmidrule(lr){15-16}
        & &AUROC$\uparrow$	&FPR95$\downarrow$	&AUROC$\uparrow$	&FPR95$\downarrow$	&AUROC$\uparrow$	&FPR95$\downarrow$	&AUROC$\uparrow$	&FPR95$\downarrow$	&AUROC$\uparrow$	&FPR95$\downarrow$	&AUROC$\uparrow$	&FPR95$\downarrow$	&AUROC$\uparrow$	&FPR95$\downarrow$
        \\
        \midrule
         \multirow{2}{*}{\centering Res18} & \XSolidBrush & 86.7 & 49.3 & 87.0 & 45.2 & 94.9 & 24.6 & 95.0 & 24.9 & 90.8 & 39.3 & 89.4 & 40.3 & 89.8  & 37.3  \\
        & \CheckmarkBold & \textbf{91.5} & \textbf{32.6} & \textbf{94.8} & \textbf{21.0} & \textbf{99.0} & \textbf{5.0} & \textbf{99.3} & \textbf{3.8} & \textbf{98.2} & \textbf{9.5} & \textbf{95.1} & \textbf{20.1} & \textbf{96.3} & \textbf{15.3}  \\
        \multirow{2}{*}{\centering Res50} & \XSolidBrush & 88.7 & 51.0 & 89.0 & 47.0 & 98.0 & 10.7  & 85.6 & 55.2 & 86.9 & 57.5 & 90.6 & 42.2  & 88.2  & 43.9  \\
        & \CheckmarkBold & \textbf{92.6} & \textbf{35.1} & \textbf{93.8} & \textbf{25.4} &\textbf{99.9} & \textbf{0.2}  & \textbf{99.5} & \textbf{2.6} & \textbf{97.3} & \textbf{12.5} & \textbf{94.9} & \textbf{22.1} & \textbf{96.3}  & \textbf{16.3}  \\
        \multirow{2}{*}{\centering WRN34} & \XSolidBrush & 87.8 & 44.2 & 87.5 & 41.1  & 93.6 & 28.1 & 91.9 & 39.9 & 85.7 & 42.6  & 88.1 & 39.9 & 89.1  & 39.3  \\ 
        & \CheckmarkBold & \textbf{93.2} & \textbf{38.4} & \textbf{93.6}
        & \textbf{30.5} & \textbf{99.8} & \textbf{0.2}  & \textbf{99.8} & \textbf{1.0}  & \textbf{97.7} & \textbf{14.6} & \textbf{95.4} & \textbf{23.8}      & \textbf{96.4}  & \textbf{18.1} \\
 
        \bottomrule
    \end{tabular}
    }
\end{table*}

\subsubsection{Ablation Study on Different Backbones}

We perform ablation studies using various backbone architectures, specifically ResNet18, ResNet50, and WideResNet34. As shown in Table~\ref{tab:backbone}, we present the AUROC values for each backbone. The results demonstrate that our method consistently enhances the performance of the trained classifier, regardless of the backbone architecture used. Additionally, we observed that using WideResNet as the backbone yielded the best AUROC performance, which aligns with the general understanding of the neural network’s capacity and expressiveness.

\subsubsection{Assumption Validity and the Influence of Backbone Strength}
\label{sec: assumption}

In the analysis presented in Table~\ref{tab: main_CIFAR-100}, we observe a marginal decline in our model’s performance on CIFAR-10 test metrics, particularly in comparison to results achieved solely using the trained model. 
We attribute this decline to a deviation from our initial Energy Barrier Assumption on Peripheral-Distribution (\Cref{ass:periphery}). 
Notably, due to the stylistic similarity between CIFAR-10 and CIFAR-100, a result of their similar data collection methods, certain augmented samples in our peripheral distribution may inadvertently overlap with CIFAR-10 samples. 
This overlap challenges our core assumption that a sufficiently large energy barrier exists to clearly differentiate ID data from peripheral-distribution samples and OOD samples.

To investigate this hypothesis, we conducted a series of experiments using ResNet architectures with varying depths. The results reveal that a stronger backbone mitigates the observed performance drop, producing notable improvements in both AUROC and FPR95 metrics. 
For example, switching from ResNet18 to ResNet34 yields a 10.1\% increase in AUROC and a 24.87\% reduction in FPR95. When using ResNet50, these improvements become even more substantial: a 12.29\% increase in AUROC and a 57.95\% decrease in FPR95.
These outcomes suggest that more powerful feature extraction allows the classifier to create a more compact clustering of ID representations, thereby enhancing the reliability of peripheral samples generated from ID data.
This observation underscores the necessity of a strong feature extraction backbone to maintain the energy barrier, reinforcing the criticality of our initial assumption.

\begin{table}[tb]
\centering
\caption{OOD Detection Performance (\%) of CIFAR-100 Classifier with Different Backbones on CIFAR-10}
\label{tab:assupmtion}
\begin{tabular}{cccc} 
\toprule
Backbone & Method & AUROC $\uparrow$ & FPR95 $\downarrow$ \\
\midrule
\multirow{2}{*}{ResNet18}
    & EBO\cite{liu2020energy}   & 79.05 & 59.21 \\
    & OEST* (Ours)  & 75.22\textsubscript{(-3.83)} & 69.94\textsubscript{(+10.73)}  \\
\midrule
\addlinespace[0.5ex]
\multirow{2}{*}{ResNet34}
    & EBO\cite{liu2020energy}   & 79.43 & 82.41 \\
    & OEST* (Ours)  & 89.53\textsubscript{(+10.10)}  & 57.54\textsubscript{(-24.87)} \\
\midrule
\addlinespace[0.5ex]
\multirow{2}{*}{ResNet50}
    & EBO\cite{liu2020energy}   & 82.27 & 91.35 \\
    & OEST* (Ours)  & 94.56\textsubscript{(+12.29)} & 33.38\textsubscript{(-57.95)} \\
\bottomrule
\end{tabular}
\end{table}

\section{Conclusion}
\label{sec:conclu}

In this work, we introduce an out-of-distribution (OOD) detection framework, OEST, which leverages the principles beneath energy-based models (EBMs) to enhance classifier robustness without substantial reliance on expensive real outlier data. 
Specifically, we generate peripheral-distribution data to offer a practical and theoretically sound solution for OOD detection; 
by employing peripheral-distribution data, OEST builds an energy barrier around in-distribution samples, consequently distinguishing them from OOD samples through a spectrum of data transformations. 
In contrast to training-based methods (\cf~\Cref{sec:2.2}), OEST solely further tunes trained models and allows efficient deployment without substantial computational demands.
Furthermore, we devise the energy-barrier loss to displace the energy-bounded loss in \cite{wu2023oest} (inducing the advanced version, OEST*), provide statistical guarantee under the EBM framework, 
and successfully improve OOD detection performance across various benchmarks.
Our experiments show that OEST* consistently outperforms baseline models across various tasks. 
We are confident that OEST will pave the way to new out-of-distribution detection and open-world object detection.

\section*{Acknowledgments}
This work is supported by the Shanghai Engineering Research Center of Intelligent Computing System (No. 19DZ2252600) and the Research Grants Council (RGC) under grant \texttt{ECS-22303424}. 
The authors also thank Prof.\ Cheng Jin for providing the computational resources that significantly contributed to the success of this research.

\appendices

\section{Proof of \texorpdfstring{\Cref{thm:periphery}}{thm:periphery}}
\label{sec:proof}

\begin{theoremappendix}
When \Cref{ass:periphery} holds, we then have
$$
E(\mtx x'; f) - E(\mtx x; f) > \gamma_\alpha
$$
holds with probability $1-\alpha$.
The out-of-distribution sample $\mtx x'$ will be guaranteed to have higher energy than a random ID sample $\mtx x$ with high probability.
\end{theoremappendix}
\begin{proof}
Inserting the augmented sample $\mtx x^+$, we first reformulate the target gap $E(\mtx x'; f) - E(\mtx x; f)$ as
\begin{align*}
\paren{E(\mtx x'; f) - E(\mtx x^+; f)} + \paren{E(\mtx x^+; f) - E(\mtx x; f)}.
\end{align*}
It then suffices to bound $E(\mtx x'; f) - E(\mtx x^+; f)$ from below, considering in \Cref{ass:periphery} we already have
\begin{align*}
E(\mtx x^+; f) - E(\mtx x; f) > B \cdot \|\mtx x' - \mtx x^+\| + \gamma_\alpha.
\end{align*}
We expand $E(\mtx x'; f) - E(\mtx x^+; f)$ as
\begin{align}
T\cdot\log \brkt{\frac{\sum_{i=1}^{C}\exp\paren{\dotp{\mtx x^+}{\mtx c_i}/T}}{\sum_{i=1}^{C}\exp\paren{\dotp{\mtx x'}{\mtx c_i}/T}}}.
\label{eqn:target}
\end{align}
For the fraction of the form $\paren{\sum a_i} / \paren{\sum b_i}$, we notice
\begin{align*}
\frac{\sum_{i=1}^C a_i}{\sum_{i=1}^C b_i} = \sum_i \paren{\frac{a_i}{b_i} \cdot \frac{b_i}{\sum_{j=1}^C b_j}},
\end{align*}
which indicates the internal fraction in Eq.~\eqref{eqn:target} is as well a weighted sum of the following positive terms 
$$\exp\paren{\dotp{\mtx x^+}{\mtx c_i}/T} / \exp\paren{\dotp{\mtx x'}{\mtx c_i}/T},$$
implying the fraction is no lower than the term above for a certain $i \in [C]$.
We thus have
\begin{align*}
E(\mtx x'; f) - E(\mtx x^+; f) &\geq T\cdot\log \frac{\exp\paren{\dotp{\mtx x^+}{\mtx c_i}/T}}{\exp\paren{\dotp{\mtx x'}{\mtx c_i}/T}} \\
&= \dotp{\mtx x^+ - \mtx x'}{\mtx c_i} \\
&\geq - \|\mtx x^+ - \mtx x'\| \|\mtx c_i\| \\
&\geq - B \cdot \|\mtx x^+ - \mtx x'\|.
\end{align*}
Combining \Cref{ass:periphery}, we can attain the claim in \Cref{thm:periphery} and the proof is complete.
\end{proof}

\bibliographystyle{IEEEtran}
\bibliography{IEEE}

\newpage

\begin{IEEEbiographynophoto}{Yifan Wu} (Student Member, IEEE) 
received the B.E. degree in intelligent science and technology from the School of Computer Engineering and Science, Shanghai University, Shanghai, China, in 2024. He is currently a research assistant at the School of Computer Science, Fudan University, Shanghai, China. His research interests include anomaly detection, out-of-distribution detection, and noisy label learning.
\end{IEEEbiographynophoto}

\begin{IEEEbiographynophoto}{Xichen Ye}
received the B.E. degree in computer science and technology with the School of Information Science and Technology, Hangzhou Normal University, Zhejiang, China, in 2021, and the M.E. degree in computer science and technology with the School of Computer Engineering and Science, Shanghai University, Shanghai, China, in 2024. He is currently a research assistant at the School of Computer Science, Fudan University, Shanghai, China. His research interests include robust machine learning within the field of computer vision.
\end{IEEEbiographynophoto}

\begin{IEEEbiographynophoto}{Songmin Dai}
received the B.E. degree in applied physics from the College of Sciences, Shanghai University, Shanghai, China, in 2016, and the Ph.D. degree in computer science and technology from the School of Computer Engineering and Science, Shanghai University, Shanghai, China, in 2021. He is currently a researcher with the Hithink RoyalFlush AI Research Institute, Hangzhou, China. His research interests include unsupervised anomaly detection and generation models within the field of computer vision.
\end{IEEEbiographynophoto}

\begin{IEEEbiographynophoto}{Dengye Pan}
received the B.E. degree from Qian Weichang College, Shanghai University, Shanghai, China, in 2022, and is currently pursuing the M.S. degree in computer science and technology with the School of Computer Engineering and Science, Shanghai University, Shanghai, China. Her research interests include out-of-distribution detection within the field of computer vision.
\end{IEEEbiographynophoto}

\begin{IEEEbiographynophoto}{Xiaoqiang Li} (Member, IEEE)
received the Ph.D. degree in computer science from Fudan University, Shanghai, China, in 2004. He is currently an Associate Professor of computer science with Shanghai University, China. He is currently an ACM Member and a Senior Member of the Chinese Computer Society. He is the Deputy Director of the Multimedia Special Committee of the Shanghai Computer Society. His current research interests include image processing, pattern recognition, computer vision, and machine learning. He has published over 100 conference and journal papers in these areas, including CVPR, ICCV, Nerurips, IEEE TCSVT, IEEE TMM, IEEE TII, IEEE TCYB, IEEE TIP, \etc
\end{IEEEbiographynophoto}

\begin{IEEEbiographynophoto}{Weizhong Zhang}
received the B.S. and Ph.D. degrees from Zhejiang University in 2012 and 2017, respectively. He is currently a tenure-track Professor with the School of Data Science, Fudan University. His research interests include sparse neural network training, robustness, and out-of-distribution generalization. He has published over 40 conference and journal papers in these areas, including ICML, Neurips, ICLR, JMLR, IEEE TPAMI, IEEE TKDE, IEEE TIP, IEEE TIT, \etc
\end{IEEEbiographynophoto}

\begin{IEEEbiographynophoto}{Yifan Chen}
received the B.S. degree from Fudan University, Shanghai, China, in 2018, and the PhD degree in Statistics from University of Illinois Urbana-Champaign in 2023. He is currently an assistant professor in computer science and math at Hong Kong Baptist University. He is broadly interested in developing efficient machine learning algorithms, encompassing both statistical and deep learning models. He has published several papers in these areas, including ICML, Neurips, KDD, \etc
\end{IEEEbiographynophoto}

\vfill

\clearpage
\onecolumn
\pagestyle{plain}  %
\setcounter{page}{1}  %

\begin{center}
\setlength{\baselineskip}{1.6\baselineskip}
{\LARGE\bf Supplementary Material for ``\mytitle''}\\
\bigskip
\end{center}

\section{Supplementary Algorithm for OOD Detection}
\label{appendix: method}

In this section, we provide a supplementary description of the OOD detection process using the parameterized energy function as detailed in the main paper. Algorithm~\ref{algo:ood} outlines the steps to compute the energy-based score for a given test image sample and to determine its in- or out-of-distribution status based on a predefined threshold.

\begin{algorithm}[!h]
\caption{OOD Detection Using the Parameterized Energy Function}
\label{algo:ood_detection}
\begin{algorithmic}[1]
\REQUIRE Classifier \( f_\theta \) with parameters \( \theta \), temperature \( T \), and threshold \( \tau \)
\STATE \textbf{Input:} Test image sample \( \mtx{x} \)
\STATE Compute the logits from the fully connected layer of \( f_\theta \), denoted as \( f_\theta^{(i)}(\mtx{x}) \) for each class \( i = 1, \dots, C \)
\STATE Compute the energy function \( E(\mtx{x}; f_\theta) \) using the logits:
\[
    E(\mtx{x}; f_\theta) = - \log \left( \sum_{i=1}^C \exp  f_\theta^{(i)}(\mtx{x})  \right)
\]
\STATE Define the score function \( s_\theta(\mtx{x}) \) as the negative energy:
\[
    s_\theta(\mtx{x}) = -E(\mtx{x}; f_\theta)
\]
\STATE Set the OOD discriminator \( D(\mtx{x}; \tau, f_\theta) \) as:
\[
    D(\mtx{x}; \tau, f_\theta) = 
    \begin{cases}
        1, & \text{if } s_\theta(\mtx{x}) > \tau \\
        0, & \text{if } s_\theta(\mtx{x}) \leq \tau
    \end{cases}
\]
\STATE \textbf{Output:} \( D(\mtx{x}; \tau, f_\theta) \), where \( D = 1 \) indicates OOD, and \( D = 0 \) indicates ID
\end{algorithmic}
\label{algo:ood}
\end{algorithm}

\section{Experiments}
\label{appendix: experiments}

\subsection{Datasets}
\label{appendix: dataset}
In this section, we provide a detailed description of all the datasets used in our experiments:

\textbf{CIFAR-10} \cite{krizhevsky2009learning}: A dataset of 60,000 color images in 10 classes, with 50,000 images used for training and 10,000 for testing. Each image is 32x32 pixels.

\textbf{CIFAR-100} \cite{krizhevsky2009learning}: A dataset of 60,000 color images, categorized into 100 classes. It consists of 50,000 training images and 10,000 test images, with 600 images per class.

\textbf{MNIST} \cite{lecun1998gradient}: A dataset consisting of 70,000 grayscale images of handwritten digits, where 60,000 images are used for training and 10,000 for testing. Each image is 28x28 pixels.

\textbf{KMNIST} \cite{clanuwat2018deep}: The Kuzushiji-MNIST (KMNIST) dataset consists of 70,000 grayscale images of handwritten Japanese characters, spanning 10 classes. It includes 60,000 images for training and 10,000 images for testing, with each image being 28x28 pixels in size.

\textbf{SVHN} \cite{svhn}: A real-world dataset containing over 600,000 images of street view house numbers. It is split into 73,257 training images and 26,032 testing images, with an additional 531,131 extra training images. The dataset contains 10 digit classes.

\textbf{Tin} \cite{le2015tiny}: Tiny ImageNet is a popular dataset derived from the larger ImageNet dataset. It consists of 110,000 color images across 200 different classes. Each image has a resolution of 64x64 pixels, which is smaller than the original ImageNet dataset.

\textbf{Textures} \cite{cimpoi2014describing}: A dataset of texture images with various surface patterns, used for evaluating models under non-object-like OOD settings.

\textbf{Place365} \cite{zhou2017places}: A scene recognition dataset containing 1.8 million images across 365 scene categories.

\textbf{EMNIST} \cite{cohen2017emnist}: The Extended MNIST (EMNIST) dataset contains 814,255 grayscale images of handwritten characters. In our experiments, we specifically denote EMNIST as the subset of the EMNIST dataset containing only handwritten English letters. 

\textbf{FMNIST} \cite{xiao2017fashion}: The Fashion MNIST dataset contains 70,000 grayscale images of 10 different fashion items, including t-shirts, trousers, and shoes. Each image is 28x28 pixels.

\textbf{LSUN} \cite{yu2015lsun}: The LSUN dataset is a large-scale dataset designed for scene understanding tasks, containing millions of labeled images across multiple scene and object categories. For our OOD experiments, we specifically use the LSUN-Crop subset, which consists of 10 scene categories. Each LSUN image has a larger resolution (typically 256x256 pixels), but for consistency in our experiments, the images are cropped to match the format of our in-distribution datasets.

\textbf{ImageNet} \cite{deng2009imagenet}: ImageNet is a large-scale dataset containing over 1.2 million color images categorized into 1,000 classes. In our experiments, we use the ImageNet-Resize subset, where the images have been resized to 32x32 pixels for consistency with other datasets. This subset is often used for OOD detection, as it provides a broad range of natural images.

\begin{table*}[ht]
    \centering
    \caption{AUROC (\%) for OOD detection performance on MNIST and SVHN. The top-1 results are in \textbf{bold}.}
    \label{tab:main_2}
    \begin{tabular}{lccccccc}
        \toprule
        &\multicolumn{3}{c}{MNIST$\rightarrow$}&\multicolumn{4}{c}{SVHN$\rightarrow$} \\
        \cmidrule(lr){2-4}\cmidrule(lr){5-8}
        Methods & FMNIST & EMNIST & CIFAR-10 & CIFAR-10 & CIFAR-100 & LSUN-Crop & ImageNet-Resize \\
        \midrule
        baseline\cite{hendrycks2017abaseline} & 97.2 & 88.3 & 99.6 & 93.8 & 93.5 & 94.5 & 93.9\\
        CEDA \cite{hein2019relu} & 99.4 & 89.5 & \textbf{100.0} & 96.0 & 95.9 & 98.4 & 95.5\\
        ACET \cite{hein2019relu} & 99.8 & 91.2 & \textbf{100.0} & 97.3 & 97.1 & \textbf{99.7} & 97.7\\
        OEST (Ours) & \textbf{100.0} & 95.7 & \textbf{100.0} & 99.4 & \textbf{99.1} & 99.5 & \textbf{99.9}\\
        OEST* (Ours) & \textbf{100.0} & \textbf{96.1} & \textbf{100.0} & \textbf{99.5} & \textbf{99.1} & \textbf{99.7} & \textbf{99.9}\\
        \bottomrule
    \end{tabular}
\end{table*}

\subsection{Experiments on MNIST and SVHN}
\label{appendix: mnist&svhn}

To validate the broad applicability of our proposed method, we also conduct experiments utilizing MNIST \cite{lecun1998gradient} and SVHN \cite{svhn} as in-distribution datasets. 
When MNIST \cite{lecun1998gradient} is used as the in-distribution dataset, FMNIST \cite{xiao2017fashion}, EMNIST \cite{cohen2017emnist}, and CIFAR-10 \cite{krizhevsky2009learning} are adopted for OOD testing.
For SVHN \cite{svhn} as the in-distribution dataset, we evaluate using CIFAR-10 \cite{krizhevsky2009learning}, CIFAR-100 \cite{krizhevsky2009learning}, LSUN-Crop \cite{yu2015lsun}, and ImageNet-Resize \cite{deng2009imagenet} as OOD datasets.

For MNIST, we first train a LeNet model \cite{lecun1998gradient} as the pretrained model using the SGD optimizer with a learning rate of 0.01, momentum of 0.9, and a weight decay of $5 \times 10^{-4}$ for 60 epochs. 
We then fine-tune this pretrained model for an additional 10 epochs, still using the SGD optimizer. During fine-tuning, the learning rate follows a cosine annealing schedule, starting at $1 \times 10^{-4}$ and decaying gradually to $1 \times 10^{-8}$. 
Four simple transformations—noise, blur, perm, and sobel—are applied to generate the peripheral-distribution data.

For SVHN, we train a ResNet-18 model \cite{he2016deep} as the pretrained model using the SGD optimizer with a learning rate of 0.1, momentum of 0.9, and a weight decay of $5 \times 10^{-4}$ for 100 epochs. 
We then fine-tune this pretrained model for an additional 10 epochs, once again using the SGD optimizer. 
During fine-tuning, the learning rate follows a cosine annealing schedule, starting at $1 \times 10^{-4}$ and gradually decaying to $1 \times 10^{-8}$. 
To generate peripheral-distribution data, we apply six simple transformations: noise, blur, perm, rotation, and sobel.

As shown in Table~\ref{tab:main_2}, our method continues to demonstrate superior performance compared to other baselines.

\subsection{Experimental Analysis}
\label{appendix: experimental analysis}

\subsubsection{Comparison with the Contrastive Training Scheme.}
We also observed that CSI performs poorly in this case, largely due to certain image categories having insignificant appearance differences even after transformations like rotation, which leads to model confusion. CSI relies on a single transformation type and treats the transformed images as negative class samples, which may not sufficiently capture the complexities of the data.
In contrast, our method introduces several innovative improvements. By integrating multiple data transformation techniques and incorporating the concept of peripheral-distribution, our approach addresses these limitations. This is reflected in the performance gains, where we significantly outperform CSI in both AUROC and FPR95.

For the scheme of \cite{tack2020csi} and \cite{Moon2022TailoringSF}, both of them aim to figure out better data augmentation operation to generate negative samples for contrastive learning. 
Specifically, \cite{tack2020csi} proposes rotation, and \cite{Moon2022TailoringSF}proposes two novel transformations, named LoRot-I and LoRot-W. 
However, ours will not be troubled by this, because every transformation can be useful in our training scheme just as shown in Table\ref{tab: data augmentation}. 
Furthermore, considering a simple effective transformation, rotation, as CSI shown in\Cref{tab: main_CIFAR-10} and rotation in \Cref{tab: data augmentation}, our training scheme is superior to CSI in all six benchmarks, which means ours is a better scheme to make full use of it. \\

\begin{figure}[htbp]
\centering
\captionsetup[subfigure]{labelformat=parens, labelsep=space, font=scriptsize}
\subfloat[OOD detection performance under different $\alpha$]{
    \label{fig: alpha}
    \includegraphics[width=0.4\textwidth]{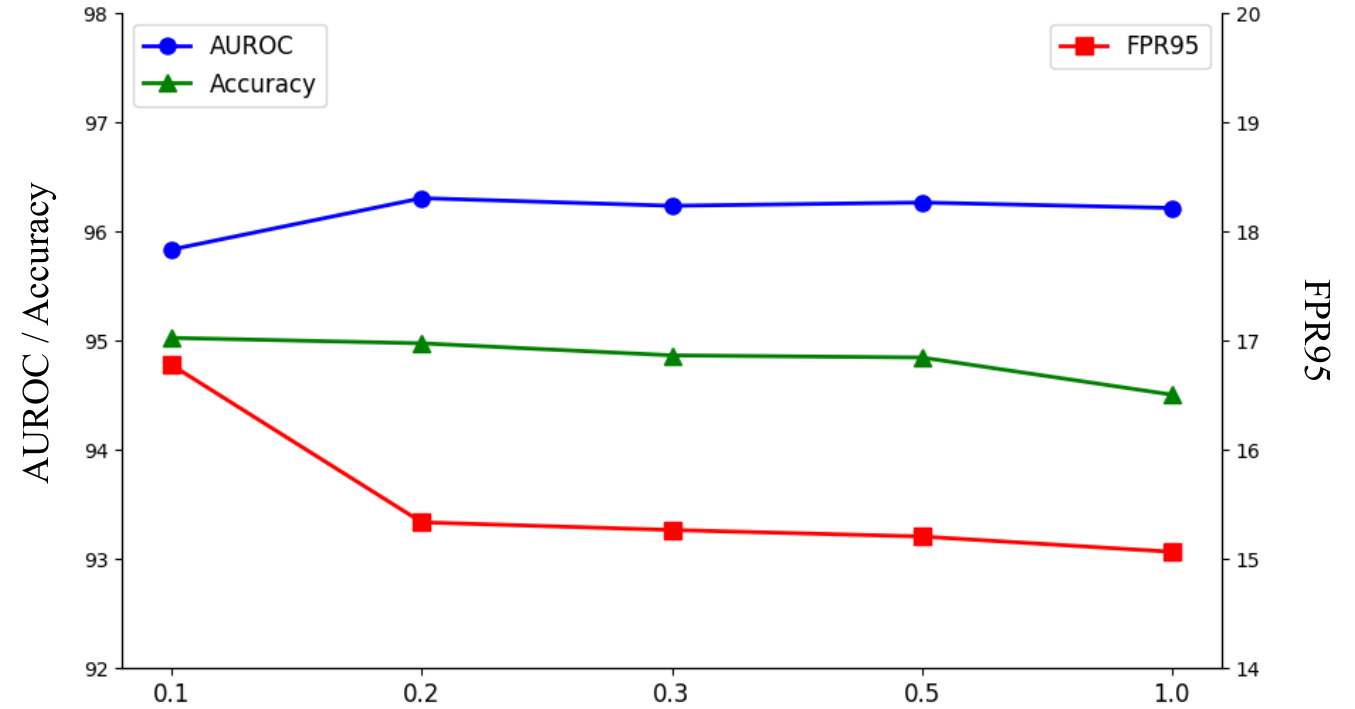}
}
\hspace{0.05\textwidth} %
\subfloat[OOD detection performance under different $\beta$]{
    \label{fig: beta}
    \includegraphics[width=0.4\textwidth]{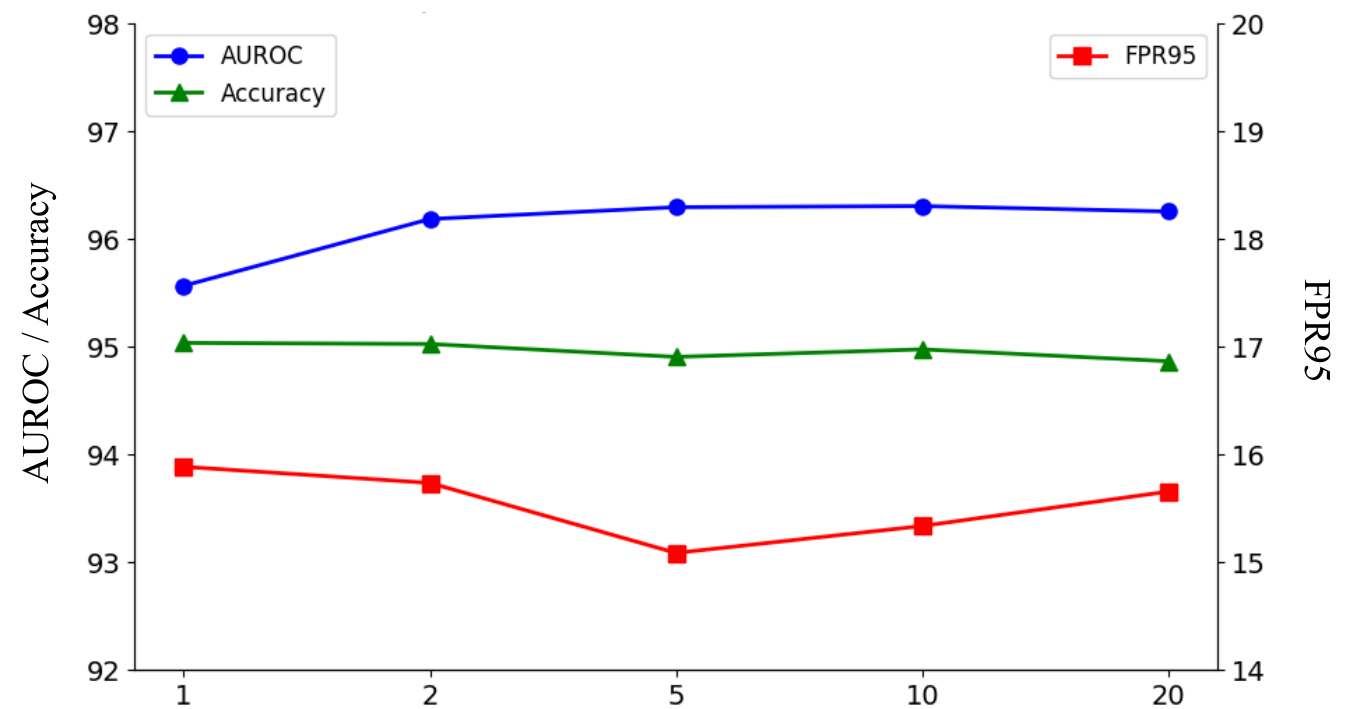}
}
\caption{
OOD detection performance of a ResNet-18 classifier trained on CIFAR-10 as the in-distribution dataset, evaluated under varying values of hyperparameters $\alpha$ and $\beta$. In (a), $\beta$ is fixed at 10, and the effect of changing $\alpha$ is shown. In (b), $\alpha$ is fixed at 0.2, and the impact of changing $\beta$ is shown. Higher AUROC and Accuracy for both experiments indicate better performance, while lower FPR95 reflects better performance.
}
\label{fig: alpha&beta}
\end{figure}

\subsubsection{Hyper-parameters Analysis}
We conducted a systematic analysis of the hyper-parameters $\alpha$ and $\beta$ to evaluate their impact on the OOD detection performance of the ResNet-18 classifier. The results are shown in Figure~\ref{fig: alpha&beta}, with AUROC and Accuracy representing the performance metrics to be maximized, and FPR95 representing the robustness metric to be minimized.

First, we analyzed the effect of $\alpha$ by fixing $\beta$ at 10, as shown in Figure~\ref{fig: alpha}. As $\alpha$ increases from 0.1 to 1.0, AUROC consistently increases and FPR95 steadily decreases, indicating that larger $\alpha$ values improve the model's ability to distinguish in-distribution (ID) and out-of-distribution (OOD) samples. However, this improvement comes at the cost of a decrease in classification accuracy for in-distribution samples, which is an undesirable side effect. To balance these competing objectives, we chose $\alpha = 0.2$ as the final value, which provides a compromise between maximizing OOD detection performance and maintaining satisfactory in-distribution accuracy.

Next, we investigated the effect of $\beta$ by fixing $\alpha$ at 0.2, as shown in Figure~\ref{fig: beta}. As $\beta$ increases from 0.05 to 10, AUROC shows an overall increasing trend, suggesting improved OOD detection capability. However, this improvement comes at the cost of a slight decrease in classification accuracy for in-distribution samples. Meanwhile, FPR95 initially decreases, reflecting enhanced robustness, but starts to increase again beyond certain values of $\beta$. To balance these effects, we chose $\beta = 10$ as the final value, which provides a reasonable trade-off between maximizing OOD detection performance and maintaining satisfactory in-distribution accuracy.

Based on these observations, we selected $\alpha = 0.2$ and $\beta = 10$ as the optimal hyper-parameter settings for our final model.

\begin{table}[htbp]
    \centering
    \caption{To verify the importance of pre-train + fine-tune process, we use CIFAR-10 as in-distribution data and rotation to generate peripheral-distribution data for fine-tune.}
    \label{tab:two_step}
    \captionsetup{skip=5pt}
    \begin{tabular}{cccccc}
    \toprule
    \centering $D^{\text{train}}_{\text{in}}$ & \makecell[c]{pre-train + fine-tune} & AUROC $\uparrow$ & FPR95 $\downarrow$  \\
    \midrule
    \multirow{2}{5em}{\centering CIFAR-10} & training from scratch & 92.1 & 31.1 \\
    & \checkmark & \textbf{96.3} & \textbf{15.3} \\
    \bottomrule
    \end{tabular}
\end{table}

\subsubsection{The Importance of Fine-Tune}
To further validate the necessity of both the pre-train and fine-tune steps, we conducted experiments as illustrated in \Cref{tab:two_step}. Compared to training from scratch, the pre-train plus fine-tune scheme yields better results, effectively enhancing the model's performance and reducing the false positive rate. 
A possible reason is that training from scratch with augmented samples together with the original samples can lead the model to prematurely learn the pattern differences between samples. However, these differences are merely low-level semantic information at the texture and color levels. In contrast, when training is divided into pre-train plus fine-tune stages, the model has already learned higher-level semantic information during the pre-train stage. This can help the model better understand the differences between the data-augmented samples and the original samples during the fine-tune stage, leading to superior out-of-distribution detection results. This experimental outcome further validates the necessity of pre-train and fine-tune in enhancing model performance.
Also, this scheme can reduces the consumption of computational resources, as we can use well-trained classifier and only need to fine-tune 10 epochs.

\subsubsection{Visualization of t-SNE}
\label{sec:t-SNE}
We performed t-SNE visualization of the features before and after fine-tune to provide a clear illustration. A shown in Figure~\ref{fig:clusted-features}, red represents the test samples of CIFAR-10, orqange represents the rotated CIFAR-10 samples, blue, purple and green represent three out-of-distribution test datasets of CIFAR-100, SVHN, and ImageNet, respectively. Feature embeddings are obtained from the ultimate convolutional layer.
We observed that rotated CIFAR-10 are located in the peripheral area of in-distribution, which supports the notion of the peripheral augmented samples lying between the in-distribution and out-of-distribution samples. Moreover, by assigning different energy scores to the rotated CIFAR-10 samples and the original samples during finetue, the decision boundaries of the classifier become much clearer and can effectively discriminate the CIFAR-100 with the highest similarity to the original distribution samples.

\end{document}